\documentclass{article}

\usepackage[nonatbib,final]{neurips_2024}

\usepackage[utf8]{inputenc} 
\usepackage[T1]{fontenc}    
\usepackage{hyperref}       
\usepackage{url}            
\usepackage{booktabs}       
\usepackage{amsfonts}       
\usepackage{nicefrac}       
\usepackage{microtype}      
\usepackage{xcolor}         

\usepackage{amsmath}
\usepackage{amsthm}
\usepackage[pdftex]{graphicx}
\usepackage{subfigure}
\usepackage{multirow}
\usepackage{colortbl}
\usepackage{paralist}
\usepackage{algorithm}
\usepackage{algpseudocode}
\usepackage{wrapfig}
\usepackage{mathtools}
\usepackage{enumitem}

\theoremstyle{plain}
\newtheorem{theorem}{Theorem}[section]
\newtheorem{proposition}[theorem]{Proposition}
\newtheorem{lemma}[theorem]{Lemma}

\theoremstyle{definition}

\theoremstyle{remark}
\newtheorem{remark}[theorem]{Remark}

\definecolor{Gray}{gray}{0.9}

\newcommand{\hh}[1]{{\small\color{red}{\bf hh: #1}}}
\newcommand{\zhe}[1]{{\small\color{orange}{\bf zhe: #1}}}
\newcommand{\yz}[1]{{\small\color{blue}{\bf yz: #1}}}
\newcommand{\mh}[1]{{\small\color{cyan}{\bf mh: #1}}}
\newcommand{\hy}[1]{{\small\color{brown}{\bf hy: #1}}}
\newcommand{\RZ}[1]{{\small\color{teal}{\bf ruizhong: #1}}}
\newcommand{\zc}[1]{{\small\color{green}{\bf zhichen: #1}}}

\newcommand{\cleancomments}{
	\renewcommand{\hh}[1]{}
	\renewcommand{\zhe}[1]{}
	\renewcommand{\yz}[1]{}
        \renewcommand{\mh}[1]{}
        \renewcommand{\hy}[1]{}
        \renewcommand{\RZ}[1]{}
        \renewcommand{\zc}[1]{}
}

\cleancomments

\title{Discrete-state Continuous-time Diffusion for Graph Generation}

\author{%
  Zhe Xu\thanks{University of Illinois  Urbana-Champaign. \{zhexu3, rq5, zhichenz, htong\}@illinois.edu}\\
  \And
  Ruizhong Qiu\footnotemark[1]\\
  \And
  Yuzhong Chen\thanks{Visa Research. \{yuzchen, hchen, xirafan, menpan, mahdas\}@visa.com}\\
  \And
  Huiyuan Chen\footnotemark[2]\\
  \And
  Xiran Fan\footnotemark[2]\\
  \And
  Menghai Pan\footnotemark[2]\\
  \And
  Zhichen Zeng\footnotemark[1]\\
  \And
  Mahashweta Das\footnotemark[2]\\
  \And
  Hanghang Tong\footnotemark[1]\\
}

\begin{document}

\maketitle

\begin{abstract}
Graph is a prevalent discrete data structure, whose generation has wide applications such as drug discovery and circuit design. Diffusion generative models, as an emerging research focus, have been applied to graph generation tasks. Overall, according to the space of {\em states} and {\em time} steps, diffusion generative models can be categorized into discrete-/continuous-state discrete-/continuous-time fashions. In this paper, we formulate the graph diffusion generation in a discrete-state continuous-time setting, which has never been studied in previous graph diffusion models. The rationale of such a formulation is to preserve the discrete nature of graph-structured data and meanwhile provide flexible sampling trade-offs between sample quality and efficiency. Analysis shows that our training objective is closely related to the generation quality and our proposed generation framework enjoys ideal invariant/equivariant properties concerning the permutation of node ordering. Our proposed model shows competitive empirical performance against state-of-the-art graph generation solutions on various benchmarks and at the same time can flexibly trade off the generation quality and efficiency in the sampling phase. 
\end{abstract}


\section{Introduction}

\begin{wrapfigure}{r}{0.55\textwidth}
  \vspace{-7mm}
  \begin{center}
  \includegraphics[width=0.5\textwidth]{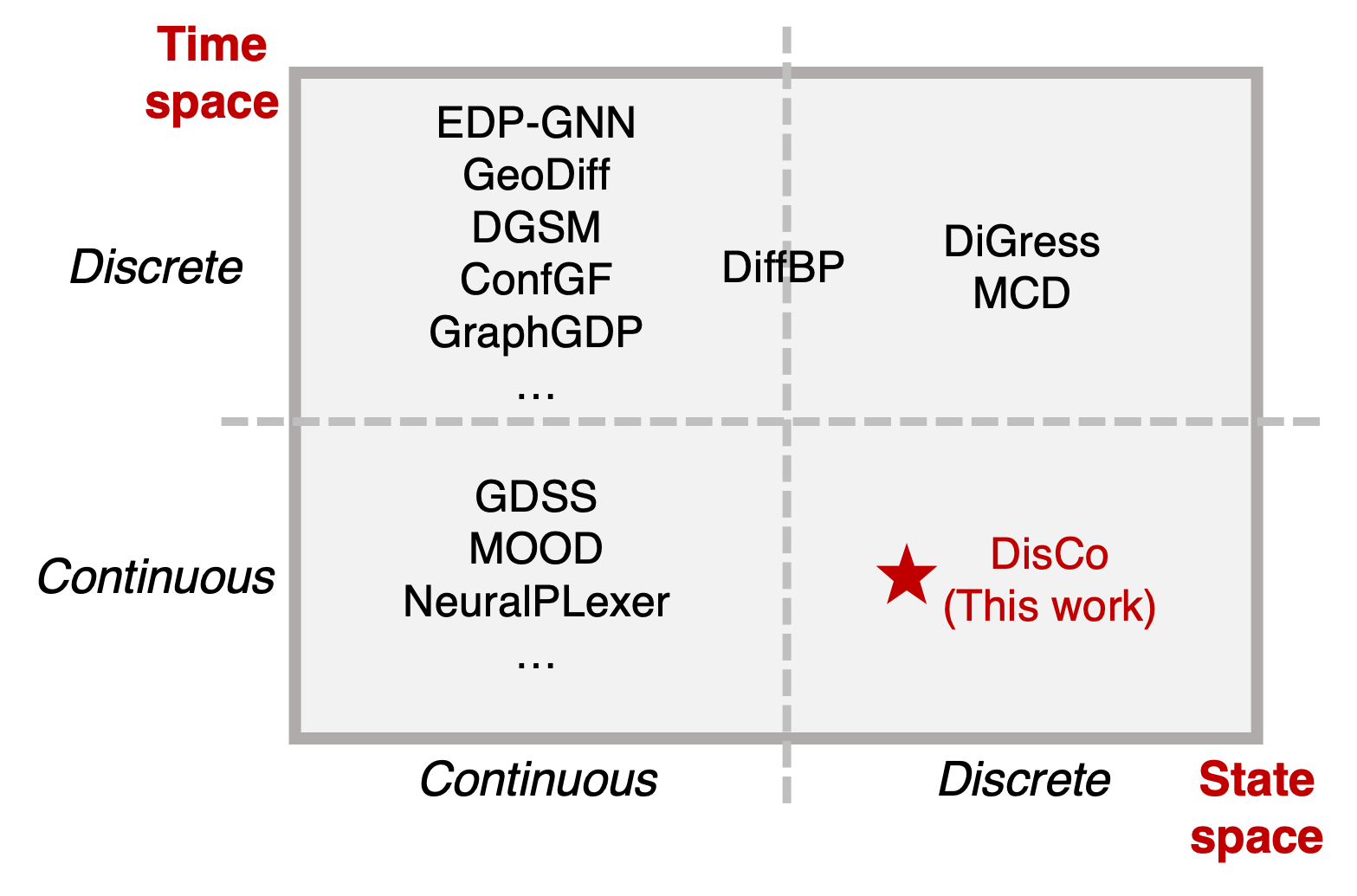}
  \end{center}
  \vspace{-5mm}
  \caption{A taxonomy of graph diffusion models.}
  \vspace{-5mm}
  \label{fig: taxonomy}
\end{wrapfigure}

Graph generation has been studied for a long time with broad applications, based on either the one-shot (i.e., one-step)~\cite{DBLP:journals/corr/abs-1905-11600,DBLP:journals/jcheminf/KwonLCSK20,DBLP:conf/aistats/NiuSSZGE20,DBLP:conf/icml/MartinkusLPW22,DBLP:conf/iclr/VignacF22,DBLP:conf/icml/JoLH22} or auto-regressive generation paradigm~\cite{DBLP:conf/icml/YouYRHL18,DBLP:conf/icml/JinBJ18,DBLP:conf/nips/LiaoLSWHDUZ19,DBLP:journals/mlst/MercadoRLKECB21}. The former generates all the graph components at once and the latter does that sequentially. A recent trend of applying diffusion generative models~\cite{DBLP:conf/icml/Sohl-DicksteinW15,DBLP:conf/nips/HoJA20,DBLP:conf/iclr/0011SKKEP21} to graph generation tasks attracts increasing attentions because of its excellent performance and solid theoretical foundation. In this paper, we follow the one-shot generation paradigm, the same as most graph diffusion generative models.

Some earlier attempts at graph diffusion models treat the graph data in a continuous state space by viewing the graph topology and features as continuous variables~\cite{DBLP:conf/aistats/NiuSSZGE20}. Such a formulation departs from the discrete nature of graph-structured data; e.g., topological sparsity is lost and the discretization in the generation process requires extra hyper-parameters. DiGress~\cite{DBLP:conf/iclr/VignacKSWCF23} is one of the early efforts applying discrete-state diffusion models to graph generation tasks and is the current state-of-the-art graph diffusion generative model. However, DiGress is defined in the discrete time space whose generation is inflexible. This is because, its number of sampling steps must match the number of forward diffusion steps, which is a fixed hyperparameter after the model finishes training. A unique advantage of the continuous-time diffusion models~\cite{DBLP:conf/iclr/0011SKKEP21, DBLP:conf/icml/JoLH22} lies in their flexible sampling process, and its simulation complexity is proportional to the number of sampling steps, determined by the step size of various numerical approaches (e.g., $\tau$-leaping~\cite{gillespie2001approximate,DBLP:conf/nips/CampbellBBRDD22,DBLP:conf/iclr/SunYDSD23}) and decoupled from the models' training. Thus, a discrete-state continuous-time diffusion model is highly desirable for graph generation tasks.


Driven by the recent advance of continuous-time Markov Chain (CTMC)-based diffusion generative model~\cite{DBLP:conf/nips/CampbellBBRDD22}, we incorporate the ideas of CTMC into the corruption and denoising of graph data and propose the first discrete-state continuous-time graph diffusion generative model. It shares the same advantages as DiGress by preserving the discrete nature of graph data and meanwhile overcomes the drawback of the nonadjustable sampling process in DiGress. This \underline{Dis}crete-state \underline{Co}ntinuous-time graph diffusion model is named \textsc{DisCo}.

\textsc{DisCo} bears several desirable properties and advantages. First, despite its simplicity, the training objective has a rigorously proved connection to the sampling error.  Second, its formulation includes a parametric graph-to-graph mapping, named backbone model, whose input-output architecture is shared between \textsc{DisCo} and DiGress. Therefore, the graph transformer (GT)-based backbone model~\cite{DBLP:journals/corr/abs-2202-08455} from DiGress can be seamlessly plugged into \textsc{DisCo}. Third, a concise message-passing neural network backbone model is explored with \textsc{DisCo}, which is simpler than the GT backbone and has decent empirical performance. Last but not least, our analyses show that the forward and reverse diffusion process in \textsc{DisCo} can retain the permutation-equivariant/invariant properties for its training loss and sampling distribution, both of which are critical and practical inductive biases on graph data.

Comprehensive experiments on plain and molecule graphs show that \textsc{DisCo} can obtain competitive or superior performance against state-of-the-art graph generative models and provide additional sampling flexibility. Our main contributions are summarized:
\begin{itemize}[topsep=0pt,noitemsep,leftmargin=*]
    \item {\bf Model.} We propose the first discrete-state continuous-time graph diffusion model, \textsc{DisCo}. We utilize the successful graph-to-graph neural network architecture from DiGress and further explore a new lightweight backbone model with decent efficacy.
    \item {\bf Analysis.} Our analysis reveals (1) the key connection between the training loss and the approximation error (Theorem~\ref{thm: correctness of objective function}) and (2) invariant/equivariant properties of \textsc{DisCo} in terms of the permutation of nodes (Theorems~\ref{thm: permutation-invariant density} and~\ref{thm: permutation-invariant training loss}).
    \item {\bf Experiment.} Extensive experiments validate the empirical performance of \textsc{DisCo}.
\end{itemize}
\section{Preliminaries}

\subsection{Discrete-State Continuous-time Diffusion Models}
\label{sec: CTMC diffusion models}

A  $D$-dimensional discrete state space is represented as $\mathcal{X} = \{1,\dots,C\}^{D}$. A continuous-time Markov Chain (CTMC) $\{\mathbf{x}_t=[x_t^1, \dotsm x_t^D]\}_{t\in[0,T]}$ is characterized by its (time-dependent) rate matrix $\mathbf{R}_t\in\mathbb{R}^{|\mathcal{X}|\times|\mathcal{X}|}$. Here $\mathbf{x}_t$ is the state at the time step $t$. The transition probability $q_{t|s}$ between from time $s$ to $t$ satisfies the Kolmogorov forward equation, for $s<t$,
\begin{align}
    \frac{d}{dt}q_{t|s}(\mathbf{x}_t|\mathbf{x}_s) = \sum_{\xi \in \mathcal{X}} q_{t|s}(\xi|\mathbf{x}_s)\mathbf{R}_t(\xi, \mathbf{x}_t),
    \label{eq: kolmogorov forward}
\end{align}
The marginal distribution can be represented as $q_t(\mathbf{x}_t) = \sum_{\mathbf{x}_0\in\mathcal{X}}q_{t|0}(\mathbf{x}_t|\mathbf{x}_0)\pi_{\texttt{data}}(\mathbf{x}_0)$
where $\pi_{\texttt{data}}(\mathbf{x}_0)$ is the data distribution.
If the CTMC is defined in time interval $[0, T]$ and if the rate matrix $\mathbf{R}_t$ is well-designed, the final distribution $q_T(\mathbf{x}_T)$ can be close to a tractable reference distribution $\pi_{\texttt{ref}}(\mathbf{x}_T)$, e.g., uniform distribution. We notate the reverse stochastic process as $\tilde{\mathbf{x}}_{t} = \mathbf{x}_{T-t}$; a well-known fact (e.g., Section 5.9 in ~\cite{resnick1992adventures}) is that the reverse process $\{\tilde{\mathbf{x}}_t\}_{t\in[0,T]}$ is also a CTMC, characterized by the reverse rate matrix: $\tilde{\mathbf{R}}_t(\mathbf{x},\mathbf{y})=\frac{q(\mathbf{y})}{q(\mathbf{x})}\mathbf{R}_t(\mathbf{y},\mathbf{x})$.
The goal of the CTMC-based diffusion models is an accurate estimation of the reverse rate matrix $\tilde{\mathbf{R}}_t$ so that new data can be generated by sampling the reference distribution $\pi_{\texttt{ref}}$ and then simulating the reverse CTMC~\cite{gillespie1976general,gillespie1977exact,gillespie2001approximate,anderson2007modified}. However, the complexity of the rate matrix is prohibitively high because there are $C^D$ possible states. A reasonable simplification is to factorize the process over dimensions~\cite{DBLP:conf/nips/CampbellBBRDD22,DBLP:conf/iclr/SunYDSD23,DBLP:conf/iclr/VignacKSWCF23,DBLP:conf/nips/AustinJHTB21}. Specifically, the forward process is factorized as $q_{t|s}(\mathbf{x}_t|\mathbf{x}_s) = \prod_{d=1}^D q_{t|s}(x_t^d|x_s^d)$, for $s<t$. Then, the forward diffusion of each dimension is independent and is governed by dimension-specific forward rate matrices $\{\mathbf{R}_t^d\}_{d=1}^{D}$. With such a factorization, the goal is to estimate the dimension-specific reverse rate matrices $\{\tilde{\mathbf{R}}^d_t\}_{d=1}^{D}$.

The dimension-specific reverse rate is represented as $\tilde{\mathbf{R}}_t^d(x^d,y^d) = \sum_{x_0^d}\mathbf{R}^d_{t}(y^d, x^d)\frac{q_{t|0}(y^d|x^d_0)}{q_{t|0}(x^d|x^d_0)}q_{0|t}(x^d_0|\mathbf{x})$. Campbell et al.~\cite{DBLP:conf/nips/CampbellBBRDD22} estimate $q_{0|t}(x^d_0|\mathbf{x})$ via a neural network $p_{\theta}$ such that $p_{\theta}(x^d_0|\mathbf{x},t)\approx q_{0|t}(x^d_0|\mathbf{x})$; Sun et al.~\cite{DBLP:conf/iclr/SunYDSD23} propose another singleton conditional distribution-based objective $\frac{p_{\theta}(y^d|\mathbf{x}^{\setminus d},t)}{p_{\theta}(x^d|\mathbf{x}^{\setminus d},t)} \approx  \frac{q(y^d|\mathbf{x}^{\setminus d})}{q(x^d|\mathbf{x}^{\setminus d})}$ whose rationale is Brook's Lemma~\cite{brook1964distinction,DBLP:conf/uai/Lyu09}.


\subsection{Graph Generation and Notations}
We study the graphs with \emph{categorical} node and edge attributes. A graph with $n$ nodes is represented by its edge type matrix and node type vector: $\mathcal{G}=(\mathbf{E},\mathbf{F})$, where $\mathbf{E} = (e^{(i,j)})_{i,j\in\mathbb{N}^{+}_{\leq n}}\in\{1,\dots, a+1\}^{n\times n}$, $\mathbf{F}=(f^i)_{i\in\mathbb{N}^{+}_{\leq n}}\in\{1,\dots, b\}^{n}$, $a$ and $b$ are the numbers of node and edge types, respectively. Notably, the absence of an edge is viewed as a special edge type, so there are $(a+1)$ edge types in total. The problem we study is graph generation where $N$ graphs $\{\mathcal{G}^i\}_{i\in\mathbb{N}^{+}_{\leq N}}$ from an inaccessible graph data distribution $\mathfrak{G}$ are given and we aim to generate $M$ graphs $\{\mathcal{G}^i\}_{i\in\mathbb{N}^{+}_{\leq M}}$ from $\mathfrak{G}$.


\section{Method}

\begin{figure*}[t!]
\vskip 0.2in
\centering
\includegraphics[width=0.85\textwidth]{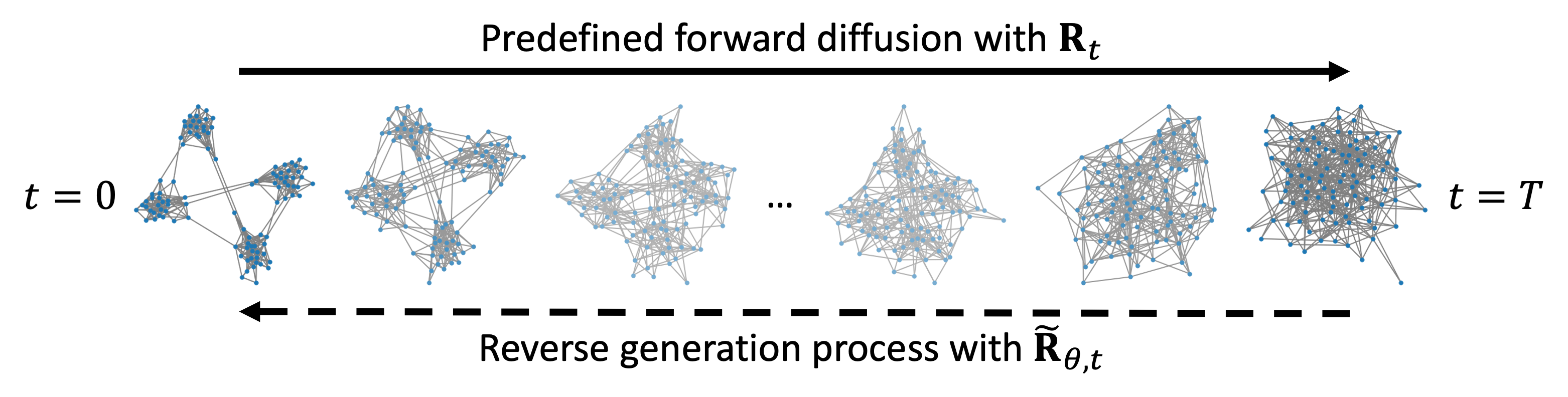}
\vspace{-2mm}
\caption{An overview of \textsc{DisCo}. A transition can happen at any time in $[0, T]$.}
\label{fig: framework}
\vskip -0.2in
\end{figure*}

This section presents the proposed discrete-state continuous-time graph diffusion model, \textsc{DisCo} whose overview is Figure~\ref{fig: framework}. Section~\ref{sec: factorized forward} introduces the necessity to factorize the diffusion process and Section~\ref{sec: forward process} details the forward process. Our training objective and its connection to sampling are introduced in Sections~\ref{sec: parameterization} and~\ref{sec: sampling reverse process}, respectively. Last but not least, a specific neural architecture of the graph-to-graph backbone model and its properties regarding the permutation of node ordering are introduced in Sections~\ref{sec: instantiation} and~\ref{sec: permutation properties}, respectively. \emph{All proofs are in Appendix.}




\subsection{Factorized Discrete Graph Diffusion Process}
\label{sec: factorized forward}
The number of possible states of an $n$-node graph is $(a+1)^{n^2}\times b^n$ which is intractably large. Thus, we follow existing discrete models~\cite{DBLP:conf/nips/AustinJHTB21,DBLP:conf/nips/CampbellBBRDD22,DBLP:conf/iclr/SunYDSD23,DBLP:conf/iclr/VignacKSWCF23} and formulate the forward processes on every node/edge to be independent. Mathematically, the forward diffusion process for $s<t$ is factorized as
\begin{align}
    q_{t|s}(\mathcal{G}_t|\mathcal{G}_s) = \prod_{i,j=1}^n q_{t|s}(e_t^{(i,j)}|e_s^{(i,j)}) \prod_{i=1}^n q_{t|s}(f_t^i|f_s^i)
    \label{eq: factorization on graphs}
\end{align}
where the edge type transition probabilities $\{q_{t|s}(e_t^{(i,j)}|e_s^{(i,j)})\}_{i,j\in\mathbb{N}^{+}_{\leq n}}$ and node type transition probabilities $\{q_{t|s}(f_t^i|f_s^i)\}_{i\in\mathbb{N}^{+}_{\leq n}}$ are characterized by their forward rate matrices $\{\mathbf{R}^{(i,j)}_t\}_{i,j\in\mathbb{N}^{+}_{\leq n}}$ and $\{\mathbf{R}^i_t\}_{i\in\mathbb{N}^{+}_{\leq n}}$, respectively. The forward processes, i.e., the forward rate matrices in our context, are predefined, which will be introduced in Section~\ref{sec: forward process}. Given the factorization of forward transition probability in Eq.~\eqref{eq: factorization on graphs}, a question is raised: \emph{what is the corresponding factorization of the forward rate matrix ($\mathbf{R}_t$) and the reverse rate matrix ($\tilde{\mathbf{R}}_{t}$)?} Remark~\ref{prop: rate factorization} shows such a factorization.


\begin{remark}
\label{prop: rate factorization}
    (Factorization of rate matrices, extended from Proposition 3 of~\cite{DBLP:conf/nips/CampbellBBRDD22}) Given the factorized forward process Eq.~\eqref{eq: factorization on graphs}, the overall rate matrices are factorized as
    \begin{align}
        \mathbf{R}_t(\bar{\mathcal{G}}, \mathcal{G}) &= \sum_i A_t^i
        + \sum_{i,j} B_t^{(i,j)} \label{eq: forward rate matrix}\\
        \tilde{\mathbf{R}}_t (\mathcal{G}, \bar{\mathcal{G}}) &= \sum_i A_t^i \sum_{f_0^i}\frac{q_{t|0}(\bar{f}^i|f^i_0)}{q_{t|0}(f^i|f^i_0)} q_{0|t}(f^i_0|\mathcal{G}) + \sum_{i,j} B_t^{(i,j)} \sum_{e_0^{(i,j)}} \frac{q_{t|0}(\bar{e}^{(i,j)}|e^{(i,j)}_0)}{q_{t|0}(e^{(i,j)}|e^{(i,j)}_0)} q_{0|t}(e^{(i,j)}_0|\mathcal{G})
        \label{eq: reverse rate matrix}
    \end{align}
    where $A_t^i = \mathbf{R}^i_t(\bar{f}^i,f^i)\delta_{\bar{\mathcal{G}}\setminus \bar{f}^i, \mathcal{G}\setminus f^i}$, $B_t^{(i,j)} = \mathbf{R}^{(i,j)}_t(\bar{e}^{(i,j)},e^{(i,j)})\delta_{\bar{\mathcal{G}}\setminus \bar{e}^{(i,j)}, \mathcal{G}\setminus e^{(i,j)}}$, the operator $\delta_{\bar{\mathcal{G}}\setminus \bar{f}^i, \mathcal{G}\setminus f^i}$ (or $\delta_{\bar{\mathcal{G}}\setminus \bar{e}^{(i,j)}, \mathcal{G}\setminus e^{(i,j)}}$) checks whether two graphs $\bar{\mathcal{G}}$ and $\mathcal{G}$ are exactly the same except for node $i$ (or the edge between nodes $i$ and $j$).
\end{remark}

Note that this factorization itself is not our contribution but a necessary part of our framework, so we mention it here for completeness. Its full derivation is in Appendix - Section~\ref{sec: rate factorization appendix}. Next, we detail the design of forward rate matrices.

\subsection{Forward Process}
\label{sec: forward process}
A proper choice of the forward rate matrices $\{\mathbf{R}^{(i,j)}_t\}_{i,j\in\mathbb{N}^{+}_{\leq n}}$ and $\{\mathbf{R}^i_t\}_{i\in\mathbb{N}^{+}_{\leq n}}$ is important because (1) the probability distributions of node and edge types, $\{q(f^i_t)\}_{i\in\mathbb{N}^{+}_{\leq n}}$ and $\{q(e^{(i,j)}_t)\}_{i,j\in\mathbb{N}^{+}_{\leq n}}$, should converge to their reference distributions within $[0, T]$ and (2) the reference distributions should be easy to sample (e.g., uniform distribution). We follow~\cite{DBLP:conf/nips/CampbellBBRDD22} to formulate $\mathbf{R}^{(i,j)}_t=\beta(t)\mathbf{R}^{(i,j)}_e,\ \forall i,j$ and $\mathbf{R}^{i}_t=\beta(t)\mathbf{R}^{i}_f,\ \forall i$, where $\beta(t)$ is a corruption schedule, $\{\mathbf{R}^{(i,j)}_e\}$ and $\{\mathbf{R}^{i}_f\}$ are the base rate matrices. For brevity, we set all the nodes/edges to share a common node/edge rate matrix, i.e., $\mathbf{R}^{(i,j)}_e=\mathbf{R}_e$ and $\mathbf{R}^{i}_f=\mathbf{R}_f$, $\forall i, j$. Then, the forward transition probability for all the nodes and edges are $q_{t|0}(f_t=v|f_0=u)=(e^{\int_{0}^t\beta(s)\mathbf{R}_f ds})_{uv}$ and $q_{t|0}(e_t=v|e_0=u)=(e^{\int_{0}^t\beta(s)\mathbf{R}_e ds})_{uv}$, respectively. We omit the superscript $i$ (or $(i,j)$) because the transition probability is shared by all the nodes (or edges). The detailed derivation of the above analytic forward transition probability is provided in Appendix - Section~\ref{sec: forward appendix}.


For categorical data, a reasonable reference distribution is a uniform distribution, i.e., $\pi_f=\frac{\mathbf{1}}{b}$ for nodes and $\pi_e=\frac{\mathbf{1}}{a+1}$ for edges. In addition, inspired by~\cite{DBLP:conf/iclr/VignacKSWCF23}, we find that node and edge marginal distributions $\mathbf{m}_f$ and $\mathbf{m}_e$ are good choices as the reference distributions. Concretely, an empirical estimation of $\mathbf{m}_f$ and $\mathbf{m}_e$ is to count the number of node/edge types and normalize them. The following proposition shows how to design the rate matrices to guide the forward process to converge to uniform and marginal distributions.

\begin{proposition}
\label{prop: rate matrices for uniform and marginal dist}
    The forward processes for nodes and edges converge to uniform distributions if $\mathbf{R}_f=\mathbf{1}\mathbf{1}^{\top}-b\mathbf{I}$ and $\mathbf{R}_e=\mathbf{1}\mathbf{1}^{\top}-(a+1)\mathbf{I}$; they converge to marginal distributions $\mathbf{m}_f$ and $\mathbf{m}_e$ if $\mathbf{R}_f=\mathbf{1}\mathbf{m}_f^{\top}-\mathbf{I}$ and $\mathbf{R}_e=\mathbf{1}\mathbf{m}_e^{\top}-\mathbf{I}$. $\mathbf{1}$ is an all-one vector and $\mathbf{I}$ is an identity matrix. 
\end{proposition}



Regarding the selection of $\beta(t)$, we follow~\cite{DBLP:conf/nips/HoJA20, DBLP:conf/iclr/0011SKKEP21, DBLP:conf/nips/CampbellBBRDD22} and set $\beta(t)=\alpha\gamma^t \log(\gamma)$ for a smooth change of the rate matrix. $\alpha$ and $\gamma$ are hyperparameters. Detailed settings are in Appendix~\ref{sec: hyperparameter settings}.

\begin{algorithm}[t]
\caption{Training of \textsc{DisCo}}
\label{alg: training}

\begin{algorithmic}[1]
\State A training graph $\mathcal{G}_0=(\{f_0^i\},\{e_0^{(i,j)}\})$ is given.
\State Sample $t\sim\mathcal{U}_{(0,T)}$; sample $\mathcal{G}_t$ based on transition probabilities $q_{t|0}(f_t=v|f_0=u)=(e^{\int_{0}^t\beta(s)\mathbf{R}_f ds})_{uv}$ and $q_{t|0}(e_t=v|e_0=u)=(e^{\int_{0}^t\beta(s)\mathbf{R}_e ds})_{uv}$, given $\mathcal{G}_0=(\{f_0^i\},\{e_0^{(i,j)}\})$.
\State Predict the clean graph $\hat{\mathcal{G}}_0=(\{\hat{f}_0^i\},\{\hat{e}_0^{(i,j)}\})\gets \Big(\{p^{\theta}_{0|t}(f^i | \mathcal{G}_t)\},
\{p^{\theta}_{0|t}(e^{(i,j)}| \mathcal{G}_t)\}\Big)$, given $\mathcal{G}_t$.
\State Compute cross-entropy loss between $\mathcal{G}_0$ and $\hat{\mathcal{G}}_0$ based on Eq.~\eqref{eq: utility loss} and update $\theta$.
\end{algorithmic}
\end{algorithm}

\subsection{Parameterization and Optimization Objective}
\label{sec: parameterization}

Next, we introduce the estimation of the reverse process from its motivation. The reverse process is essentially determined by the reverse rate matrix $\tilde{\mathbf{R}}_t$ in Eq.~\eqref{eq: reverse rate matrix}, whose computation needs $q_{0|t}(f^i_0|\mathcal{G})$ and $q_{0|t}(e^{(i,j)}_0|\mathcal{G})$, $\forall i,j$; their exact estimation is expensive because according to Bayes' rule, $p_t(\mathcal G)$ is needed, whose computation needs to enumerate all the given graphs: $p_t(\mathcal G) = \sum_{\mathcal G_0}q_{t|0}(\mathcal G|\mathcal G_0)\pi_{\texttt{data}}(\mathcal G_0)$.


Thus, we propose parameterizing the reverse transition probabilities via a neural network $\theta$ whose specific architecture is introduced in Section~\ref{sec: instantiation}. The terms $\{q_{0|t}(f^i_0|\mathcal{G})\}_{i\in\mathbb{N}^{+}_{\leq n}}$ and  $\{q_{0|t}(e^{(i,j)}_0|\mathcal{G})\}_{i,j\in\mathbb{N}^{+}_{\leq n}}$ in Eq.~\eqref{eq: reverse rate matrix} are replaced with the parameterized $\{p^{\theta}_{0|t}(f^i | \mathcal{G})\}_{i\in\mathbb{N}^{+}_{\leq n}}$ and $\{p^{\theta}_{0|t}(e^{(i,j)}| \mathcal{G})\}_{i,j\in\mathbb{N}^{+}_{\leq n}}$. Thus, a parameterized reverse rate matrix $\tilde{\mathbf{R}}_{\theta,t} (\mathcal{G}, \bar{\mathcal{G}})$ is represented as $\tilde{\mathbf{R}}_{\theta,t} (\mathcal{G}, \bar{\mathcal{G}}) = \sum_i \tilde{\mathbf{R}}^i_{\theta,t}(f^i,\bar{f}^i) + \sum_{i,j} \tilde{\mathbf{R}}^{(i,j)}_{\theta,t}(e^{(i,j)},\bar{e}^{(i,j)})$
where $\tilde{\mathbf{R}}^i_{\theta,t}(f^i,\bar{f}^i) = A_t^i \sum_{f_0^i}\frac{q_{t|0}(\bar{f}^i|f^i_0)}{q_{t|0}(f^i|f^i_0)} p^{\theta}_{0|t}(f^i_0| \mathcal{G})$, $\tilde{\mathbf{R}}^{(i,j)}_{\theta,t}(e^{(i,j)},\bar{e}^{(i,j)})=B_t^{(i,j)} \sum_{e_0^{(i,j)}}\frac{q_{t|0}(\bar{e}^{(i,j)}|e^{(i,j)}_0)}{q_{t|0}(e^{(i,j)}|e^{(i,j)}_0)} p^{\theta}_{0|t}(e^{(i,j)}_0| \mathcal{G})$, and the remaining notations are the same as Eq.~\eqref{eq: reverse rate matrix}. Note that all the terms $\{p^{\theta}_{0|t}(f^i | \mathcal{G})\}_{i\in\mathbb{N}^{+}_{\leq n}}$ and $\{p^{\theta}_{0|t}(e^{(i,j)}| \mathcal{G})\}_{i,j\in\mathbb{N}^{+}_{\leq n}}$ can be viewed together as a graph-to-graph mapping $\theta: \mathfrak G \mapsto \mathfrak G$, whose input is the noisy graph $\mathcal{G}_t$ and its output is the predicted clean graph probabilities, concretely, the node/edge type probabilities of all the nodes and edges.

Intuitively, the discrepancy between the groundtruth $\tilde{\mathbf{R}}_{t}$ (from Eq.~\eqref{eq: reverse rate matrix}) and the parametric $\tilde{\mathbf{R}}_{\theta,t}$ should be small. Theorem~\ref{thm: correctness of objective function} establishes a cross-entropy (CE)-based upper bound of such a discrepancy, where the estimated probability vectors (sum is $1$) are notated as $\hat{f}^i_0 = [p^{\theta}_{0|t}(f^i=1 | \mathcal{G}_t),\dots, p^{\theta}_{0|t}(f^i=b | \mathcal{G}_t)]^\top\in[0,1]^b$ and $\hat{e}^{(i,j)}_0=[p^{\theta}_{0|t}(e^{(i,j)}=1| \mathcal{G}_t), \dots, p^{\theta}_{0|t}(e^{(i,j)}=a+1| \mathcal{G}_t)]^\top \in[0,1]^{a+1}$.
\begin{theorem}[Approximation error] for $\mathcal{G} \neq \bar{\mathcal{G}}$
    \label{thm: correctness of objective function}
    \begin{align}
        \Big| \tilde{\mathbf{R}}_{t}(\mathcal{G},\bar{\mathcal{G}})-\tilde{\mathbf{R}}_{\theta,t}(\mathcal{G},\bar{\mathcal{G}}) \Big|^2 &\leq C_t + C^{\texttt{node}}_t \mathbb E_{\mathcal G_0} q_{t|0}(\mathcal G|\mathcal G_0)\sum_i\mathcal{L}_{\mathtt{CE}}\big(\mathtt{One\text-Hot}(f_0^i), \hat{f}_0^i\big)\nonumber \\ 
        &+ C^{\texttt{edge}}_t\mathbb E_{\mathcal G_0} q_{t|0}(\mathcal G|\mathcal G_0)\sum_{i,j}\mathcal{L}_{\mathtt{CE}}\big(\mathtt{One\text-Hot}(e_0^{(i,j)}), \hat{e}_0^{(i,j)}\big)\label{eq: upper bound of R discrepancy}
    \end{align}
    where $C_t$, $C^{\texttt{node}}_t$, and $C^{\texttt{edge}}_t$ are constants independent on $\theta$ but dependent on $t$, $\mathcal G$, and $\bar{\mathcal G}$;  $\mathtt{One\text-Hot}$ transforms $f_0^{i}$ and $e_0^{(i,j)}$ into one-hot vectors.
\end{theorem}
The bound in Theorem~\ref{thm: correctness of objective function} is tight, i.e., the right-hand side of Eq.~\eqref{eq: upper bound of R discrepancy} is $0$, whenever $\hat{f}_0^i = q_{0|t}(f_0^i|\mathcal G_t), \forall i$ and $\hat{e}_0^{(i,j)} = q_{0|t}(e_0^{(i,j)}|\mathcal G_t), \forall i,j$. Guided by Theorem~\ref{thm: correctness of objective function}, we (1) take expectation of $t$ by sampling $t$ from a uniform distribution $t\sim\mathcal{U}_{(0,T)}$ and (2) simplify the right-hand side of Eq.~\eqref{eq: upper bound of R discrepancy} by using the unweighted CE loss as our training objective:
\begin{align}
    \min_\theta\   T\mathbb{E}_{t}\mathbb{E}_{\mathcal{G}_0}\mathbb{E}_{q_{t|0}(\mathcal{G}_t|\mathcal{G}_0)}\Big[\sum_i \mathcal{L}_{\mathtt{CE}}(\mathtt{One\text-Hot}(f_0^i),\hat{f}_0^i) +  \sum_{i,j} \mathcal{L}_{\mathtt{CE}}(\mathtt{One\text-Hot}(e_0^{(i,j)}),\hat{e}_0^{(i,j)})\Big] \label{eq: utility loss}
\end{align}
A step-by-step training algorithm is in Algorithm~\ref{alg: training}. Note that the above CE loss has been used in some diffusion models (e.g.,~\cite{DBLP:conf/nips/AustinJHTB21, DBLP:conf/nips/CampbellBBRDD22}) but lacks a good motivation, especially in the continuous-time setting. We motivate it based on the rate matrix discrepancy, as a unique contribution of this paper.

\subsection{Sampling Reverse Process}
\label{sec: sampling reverse process}

Given the parametric reverse rate matrix $\tilde{\mathbf{R}}_{\theta,t} (\mathcal{G}, \bar{\mathcal{G}})$, the graph generation process can be implemented by two steps: (1) sampling the reference distribution $\pi_{\texttt{ref}}$ (i.e., $\pi_f$ for nodes and $\pi_e$ for edges) and (2) numerically simulating the CTMC from time $T$ to $0$. The exact simulation of a CTMC has been studied for a long time, e.g.,~\cite{gillespie1976general,gillespie1977exact,anderson2007modified}. However, their simulation strategies only allow one transition (e.g., one edge/node type change) per step, which is highly inefficient for graphs as the number of nodes and edges is typically large; once a(n) node/edge is updated, $\tilde{\mathbf{R}}_{\theta,t}$ requires recomputation. A practical approximation is to assume $\tilde{\mathbf{R}}_{\theta,t}$ is fixed during a time interval $[t-\tau,t]$, i.e., delaying the happening of transitions in $[t-\tau,t]$ and triggering them all together at the time $t-\tau$; this strategy is also known as $\tau$-leaping~\cite{gillespie2001approximate,DBLP:conf/nips/CampbellBBRDD22,DBLP:conf/iclr/SunYDSD23}, and \textsc{DisCo} adopts it.

We elaborate on $\tau$-leaping for transitions of node types; the transitions of edge types are similar. The rate matrix of the $i$-th node is fixed as $\tilde{\mathbf{R}}^i_{\theta,t}(f^i,\bar{f}^i) = \mathbf{R}_t^i(\bar{f}^i,f^i) \sum_{f_0^i}\frac{q_{t|0}(\bar{f}^i|f^i_0)}{q_{t|0}(f^i|f^i_0)} p^{\theta}_{0|t}(f^i| \mathcal{G}_t)$, during $[t-\tau,t]$. According to the definition of rate matrix, in $[t-\tau,t]$, the number of transitions from $f^i$ to $\bar{f}^i$, namely $J_{f^i,\bar{f}^i}$, follows the Poisson distribution, i.e., $J_{f^i,\bar{f}^i}\sim\textrm{Poisson}(\tau \tilde{\mathbf{R}}^i_{\theta,t}(f^i,\bar{f}^i))$. For categorical data (e.g., node type), multiple transitions in $[t-\tau,t]$ are invalid and meaningless. In other words, for the $i$-th node, if the total number of transitions $\sum_{\bar{f}^i}J_{f^i,\bar{f}^i}>1$, $f^i$ keeps unchanged in $[t-\tau,t]$; otherwise, if $\sum_{\bar{f}^i}J_{f^i,\bar{f}^i}=1$ and $J_{f^i,s}=1$, i.e., there is exact $1$ transition, $f^i$ jumps to $s$. A step-by-step sampling algorithm (Algorithm~\ref{alg: tau leaping}) is in Appendix.

\begin{remark}
\label{remark: tau-leaping error}
     The sampling error of $\tau$-leaping is linear to $C_\mathtt{err}$~\cite{DBLP:conf/nips/CampbellBBRDD22}, the approximation error of the reverse rates: $\sum_{\mathcal{G}\neq \bar{\mathcal{G}}} \Big| \tilde{\mathbf{R}}_{t}(\mathcal{G},\bar{\mathcal{G}})-\tilde{\mathbf{R}}_{\theta,t}(\mathcal{G},\bar{\mathcal{G}}) \Big|\leq C_\mathtt{err}$. Interested readers are referred to Theorem 1 from~\cite{DBLP:conf/nips/CampbellBBRDD22}. Our Theorem~\ref{thm: correctness of objective function} shows the connection between our training loss and $C_\mathtt{err}$, which further verifies the correctness of our training loss.
\end{remark}

\subsection{Model Instantiation}
\label{sec: instantiation}
As mentioned in Section~\ref{sec: parameterization}, the parametric backbone $p^{\theta}_{0|t}(\mathcal{G}_0 | \mathcal{G}_t)$ is a graph-to-graph mapping whose input is the noisy graph $\mathcal{G}_t$ and its output is the predicted denoised graph $\mathcal{G}_0$. There exists a broad range of neural network architectures. Notably, DiGress~\cite{DBLP:conf/iclr/VignacKSWCF23} uses a graph Transformer (GT) as $p^{\theta}_{0|t}$, a decent reference for our continuous-time framework. We name our model with the GT backbone as \textsc{DisCo}-GT and its detailed configuration is in Appendix~\ref{sec: hyperparameter settings}. The main advantage of the GT is its long-range interaction thanks to the complete self-attention graph; however, the architecture is very complex and includes multi-head self-attention modules, leading to expensive computation.


Beyond GTs, in this paper, we posit that a regular message-passing neural network (MPNN)~\cite{DBLP:conf/icml/GilmerSRVD17} should be a promising choice for $p^{\theta}_{0|t}(\mathcal{G}_0 | \mathcal{G}_t)$. It is recognized that the MPNNs' expressiveness might not be as good as GTs'~\cite{DBLP:conf/nips/KimNMCLLH22, DBLP:conf/icml/CaiHY023}, e.g., in terms of long-range interactions. However, in our setting, the absence of an edge is viewed as a special type of edge and the whole graph is complete; therefore, such a limitation of MPNN is naturally mitigated, which is verified by our empirical evaluations.

Concretely, an MPNN-based graph-to-graph mapping is presented as follows, and \textsc{DisCo} with MPNN backbone is named \textsc{DisCo}-MPNN. Given a graph $\mathcal{G}=(\mathbf{E},\mathbf{F})$, where $\mathbf{E}\in\{1,\dots, a,a+1\}^{n\times n}$, $\mathbf{F}\in\{1,\dots, b\}^{n}$, we first transform both the matrix $\mathbf{E}$ and $\mathbf{F}$ into one-hot embeddings $\mathbf{E}_{\texttt{OH}}\in\{0,1\}^{n\times n\times (a+1)}$ and $\mathbf{F}_{\texttt{OH}}\in\{0,1\}^{n\times b}$. 
Then, some auxiliary features (e.g., the \# of specific motifs) are extracted: $\mathbf{F}_{\texttt{aux}}, \mathbf{y}_{\texttt{aux}}=\mathtt{Aux}(\mathbf{E}_{\texttt{OH}})$ to overcome the expressiveness limitation of MPNNs~\cite{DBLP:conf/nips/Chen0VB20}. Here $\mathbf{F}_{\texttt{aux}}$ and $\mathbf{y}_{\texttt{aux}}$ are the node and global auxiliary features, respectively. Note that a similar auxiliary feature engineering is also applied in DiGress~\cite{DBLP:conf/iclr/VignacKSWCF23}. More details about the $\mathtt{Aux}$ can be found in Appendix~\ref{sec: auxiliary features}. Then, three multi-layer perceptrons (MLPs) are used to map node features $\mathbf{F}_{\texttt{OH}}\oplus\mathbf{F}_{\texttt{aux}}$, edge features $\mathbf{E}_{\texttt{OH}}$, and global features $\mathbf{y}_{\texttt{aux}}$ into a common hidden space as $\mathbf{F}_{\texttt{hidden}} = \mathtt{MLP}(\mathbf{F}_{\texttt{OH}}\oplus\mathbf{F}_{\texttt{aux}})$, $\mathbf{E}_{\texttt{hidden}} = \mathtt{MLP}(\mathbf{E}_{\texttt{OH}})$, $\mathbf{y}_{\texttt{hidden}} = \mathtt{MLP}(\mathbf{y}_{\texttt{aux}})$, where $\oplus$ is a concatenation operator. The following formulas present the update of node embeddings (e.g., $\mathbf{r}^i=\mathbf{F}(i,:)$), edge embedding (e.g., $\mathbf{r}^{(i,j)}=\mathbf{E}(i,j,:)$), and global embedding $\mathbf{y}$ in an MPNN layer, where we omit the subscript $\texttt{hidden}$ if it does not cause ambiguity:
\begin{align}
    \mathbf{r}^i &\gets \mathtt{FiLM}\bigg(\mathtt{FiLM}\bigg(\mathbf{r}^i, \mathtt{MLP}\bigg(\sum_{j=1}^n \mathbf{r}^{(j,i)}/n\bigg)\bigg), \mathbf y\bigg),\ \ \mathbf{r}^{(i,j)} \gets \mathtt{FiLM} \big( \mathtt{FiLM} (\mathbf{r}^{(i,j)}, \mathbf{r}^{i} \odot \mathbf{r}^{j}), \mathbf{y}\big)\label{eq: update edge},\\
    \mathbf{y} &\gets \mathbf{y} + \mathtt{PNA}\big(\{\mathbf{r}^{i}\}_{i=1}^{n}\big) + \mathtt{PNA}\big(\{\mathbf{r}^{(i,j)}\}_{i,j=1}^{n}\big)\label{eq: update global}.
\end{align}
The edge embeddings are aggregated by mean pooling (i.e., $\sum_{j=1}^n \mathbf{r}^{(j,i)}/n$); the node pair embeddings are passed to edges by Hadamard product (i.e., $\mathbf{r}^{i} \odot \mathbf{r}^{j}$); edge/node embeddings are merged to the global embedding $\mathbf y$ via the PNA module~\cite{DBLP:conf/nips/CorsoCBLV20}; Some FiLM modules~\cite{DBLP:conf/aaai/PerezSVDC18} are used for the interaction between node/edge/global embeddings. More details about the PNA and FiLM are in Appendix~\ref{sec: auxiliary features}. In this paper, we name Eqs.~\eqref{eq: update edge} and~\eqref{eq: update global} on all nodes/edges together as an MPNN layer, $\mathbf{F},\mathbf{E},\mathbf{y}\gets\mathtt{MPNN}(\mathbf{F},\mathbf{E},\mathbf{y})$. Stacking multiple MPNN layers leads to larger model capacity. Finally, two readout MLPs are used to project the node/edge embeddings into input dimensions, $\mathtt{MLP}(\mathbf{F})\in\mathbb{R}^{n
\times b}$ and $\mathtt{MLP}(\mathbf{E})\in\mathbb{R}^{n
\times n\times (a+1)}$, which are output after wrapped with $\mathtt{softmax}$.

Both the proposed MPNN and the GT from DiGress~\cite{DBLP:conf/iclr/VignacKSWCF23} use the PNA and FiLM to merge embeddings, but MPNN does not have multi-head self-attention layers so that the computation overhead is lower.

\subsection{Permutation Equivariance and Invariance}
\label{sec: permutation properties}

Reordering the nodes keeps the property of a given graph, which is known as permutation invariance. In addition, for a given function if its input is permuted and its output is permuted accordingly, such a behavior is known as permutation equivariance. In this subsection, we analyze permutation-equivariance/invariance of the (1) diffusion framework (Lemmas~\ref{lemma: permutation equivalent layer},~\ref{lemma: permutation-invariant rate}, and~\ref{lemma: permutation-invariant transition}), (2) sampling density (Theorem~\ref{thm: permutation-invariant density}), and (3) training loss (Theorem~\ref{thm: permutation-invariant training loss}).

\begin{lemma} [Permutation-equivariant layer]
\label{lemma: permutation equivalent layer}
    The proposed MPNN layer (Eqs.~\eqref{eq: update edge} and~\eqref{eq: update global}) is permutation-equivariant.
\end{lemma}

The auxiliary features from the $\mathtt{Aux}$ are also permutation-equivariant (see Appendix~\ref{sec: auxiliary features}). Thus, the whole MPNN-based backbone $p^{\theta}_{0|t}$ is permutation-equivariant. Note that the GT-based backbone from DiGress~\cite{DBLP:conf/iclr/VignacKSWCF23} is also permutation-equivariant whose proof is omitted as it is not our contribution. Next, we show the permutation invariance of the rate matrices.

\begin{lemma}[Permutation-invariant rate matrices]
\label{lemma: permutation-invariant rate}
    The forward rate matrix of \textsc{DisCo} is permutation-invariant if it is factorized as Eq.~\eqref{eq: forward rate matrix}. The parametric reverse rate matrix of \textsc{DisCo} ($\tilde{\mathbf{R}}_{\theta,t}$) is permutation-invariant whenever the graph-to-graph backbone $p^{\theta}_{0|t}$ is permutation-equivariant.
\end{lemma}

\begin{lemma} [Permutation-invariant transition probability]
\label{lemma: permutation-invariant transition}
    For CTMC satisfying the Kolmogorov forward equation (Eq.~\eqref{eq: kolmogorov forward}), if the rate matrix is permutation-invariant (i.e., $\mathbf R_t(\mathbf x_i,\mathbf x_j) = \mathbf R_t(\mathcal P (\mathbf x_i),\mathcal P (\mathbf x_j))$, the transition probability is permutation-invariant (i.e., $q_{t|s}(\mathbf{x}_t|\mathbf{x}_s) = q_{t|s}(\mathcal P(\mathbf{x}_t)|\mathcal P(\mathbf{x}_s))$, where $\mathcal P$ is a permutation.
\end{lemma}

Based on Lemmas~\ref{lemma: permutation-invariant rate} and~\ref{lemma: permutation-invariant transition}, \textsc{DisCo}'s parametric reverse transition probability is permutation-invariant. The next theorem shows the permutation-invariance of the sampling probability.
\begin{theorem}[Permutation-invariant sampling probability]
\label{thm: permutation-invariant density}
    If both the reference distribution $\pi_{\texttt{ref}}$ and the reverse transition probability are permutation-invariant, the parametric sampling distribution $p^{\theta}_0(\mathcal{G}_0)$ is permutation-invariant.
\end{theorem}

In addition, the next theorem shows the permutation invariance of the training loss.

\begin{theorem}[Permutation-invariant training loss]
\label{thm: permutation-invariant training loss}
    The proposed training loss Eq.~\eqref{eq: utility loss} is invariant to any permutation of the input graph $\mathcal{G}_0$ if $p^{\theta}_{0|t}$ is permutation-equivariant.
\end{theorem}


\section{Experiments}
\label{sec: experiments}
This section includes: an effectiveness evaluation on plain graphs  (Section~\ref{sec: plain graph experiments}) and molecule graphs (Section~\ref{sec: molecule graph experiments}), an efficiency study (Section~\ref{sec: efficiency study}), and an ablation study (Section~\ref{sec: ablation study}). Detailed settings (Sections~\ref{sec: hardware and implmentations}-\ref{sec: hyperparameter settings}), additional effectiveness evaluation (Sections~\ref{sec: plain graph results appendix}, additional ablation study (Section~\ref{sec: ablation study appendix}), convergence study (Section~\ref{sec: convergence study}), and visualization (Section~\ref{sec: visualization}) are in Appendix. Our code is released~\footnote{\url{https://github.com/pricexu/DisCo}}.


\subsection{Plain Graph Generation}
\label{sec: plain graph experiments}
\textbf{Datasets and metrics.}
Datasets SBM, Planar~\cite{DBLP:conf/icml/MartinkusLPW22}, and Community~\cite{DBLP:conf/icml/YouYRHL18} are used. The relative squared Maximum Mean Discrepancy (MMD) for degree distributions (Deg.), clustering coefficient distributions (Clus.), and orbit counts (Orb.) distributions (the number of occurrences of substructures with $4$ nodes), Uniqueness(\%), Novelty(\%), and Validity(\%) are chosen as metrics. Details about the datasets, metrics, baselines (Section~\ref{sec: detailed settings on plain data}), and results on Community (Table~\ref{tab: effectiveness comparison on community}) are in Appendix.


\textbf{Results. }Table~\ref{tab: effectiveness comparison on plain graphs} shows the effectiveness evaluation on SBD and Planar from which we observe:
\begin{itemize}[topsep=0pt,noitemsep,leftmargin=*]
    \item \textsc{DisCo}-GT can obtain competitive performance against the SOTA, DiGress, which is reasonable because both models share the graph Transformer backbone. Note that DiGress's performance in terms of Validity is not the statistics reported in the paper but from their latest model checkpoint~\footnote{\url{https://github.com/cvignac/DiGress/blob/main/README.md}}. In fact, we found it very hard for DiGress and \textsc{DisCo}-GT to learn to generate valid SBM/Planar graphs. These two datasets have only $200$ graphs, but sometimes only after $>10,000$ epochs training, the Validity percentage can be $>50\%$. Additionally, \textsc{DisCo}-GT provides extra flexibility during sampling by adjusting the $\tau$. This is important: our models can still trade-off between the sampling efficiency and quality even after the model is trained and frozen. 
    \item In general, \textsc{DisCo}-MPNN has competitive performance against \textsc{DisCo}-GT in terms of Deg., Clus., and Orb. However, its performance is worse compared to \textsc{DisCo}-GT in terms of Validity, which might be related to the different model expressiveness. Studying the graph-to-graph model expressiveness would be an interesting future direction, e.g., generating valid Planar graphs.
\end{itemize}

\begin{table}[t!]
\centering
\caption{Performance (mean$\pm$std) on SBM and Planar datasets.}
\label{tab: effectiveness comparison on plain graphs}
\resizebox{0.96\textwidth}{!}{%
\begin{tabular}{llllllll}
\toprule
Dataset & Model & Deg.$\downarrow$ & Clus.$\downarrow$ & Orb.$\downarrow$ & Unique $\uparrow$ & Novel $\uparrow$ & Valid $\uparrow$ \\\midrule
\multirow{9}{*}{SBM} & GraphRNN~\cite{DBLP:conf/icml/YouYRHL18}  & 6.9 & 1.7 & 3.1 & \textbf{100.0}   & \textbf{100.0}  & 5.0\\
& GRAN~\cite{DBLP:conf/nips/LiaoLSWHDUZ19}  & 14.1 & 1.7 & 2.1 & \textbf{100.0}   & \textbf{100.0}  & 25.0\\
& GG-GAN~\cite{krawczuk2020gg}    & 4.4 & 2.1 & 2.3 & \textbf{100.0}   & \textbf{100.0}  & 0.0\\
& MolGAN~\cite{DBLP:journals/corr/abs-1805-11973}    & 29.4 & 3.5 & 2.8 & 95.0    & \textbf{100.0}  & 10.0\\
& SPECTRE~\cite{DBLP:conf/icml/MartinkusLPW22}   & 1.9 & 1.6 & \textbf{1.6} & \textbf{100.0}   & \textbf{100.0}  & 52.5\\
& ConGress~\cite{DBLP:conf/iclr/VignacKSWCF23}  & 34.1 & 3.1 & 4.5 & 0.0 & 0.0 & 0.0\\
& DiGress~\cite{DBLP:conf/iclr/VignacKSWCF23}   & 1.6 & 1.5 & 1.7 & \textbf{100.0}   & \textbf{100.0}  & \textbf{67.5}\\
& \cellcolor{Gray}\textsc{DisCo}-MPNN & \cellcolor{Gray}1.8\scriptsize$\pm$0.2 & \cellcolor{Gray}\textbf{0.8\scriptsize$\pm$0.1} & \cellcolor{Gray}2.7\scriptsize$\pm$0.4 & \cellcolor{Gray}\textbf{100.0\scriptsize$\pm$0.0}  & \cellcolor{Gray}\textbf{100.0\scriptsize$\pm$0.0} & \cellcolor{Gray}41.9\scriptsize$\pm$2.2  \\
& \cellcolor{Gray}\textsc{DisCo}-GT   & \cellcolor{Gray}\textbf{0.8\scriptsize$\pm$0.2} & \cellcolor{Gray}\textbf{0.8\scriptsize$\pm$0.4} & \cellcolor{Gray}2.0\scriptsize$\pm$0.5 & \cellcolor{Gray}\textbf{100.0\scriptsize$\pm$0.0}  & \cellcolor{Gray}\textbf{100.0\scriptsize$\pm$0.0} & \cellcolor{Gray}66.2\scriptsize$\pm$1.4  \\\midrule
\multirow{9}{*}{Planar} & GraphRNN~\cite{DBLP:conf/icml/YouYRHL18}  & 24.5 & 9.0 & 2508.0 & \textbf{100.0}  & \textbf{100.0}  & 0.0\\
& GRAN~\cite{DBLP:conf/nips/LiaoLSWHDUZ19}      & 3.5 & 1.4 & 1.8 & 85.0   & 2.5    & \textbf{97.5}\\
& GG-GAN~\cite{krawczuk2020gg}    & 315.0 & 8.3 & 2062.6 & \textbf{100.0}  & \textbf{100.0}  & 0.0\\
& MolGAN~\cite{DBLP:journals/corr/abs-1805-11973}    & 4.5 & 10.2 & 2346.0 & 25.0   & \textbf{100.0}  & 0.0\\
& SPECTRE~\cite{DBLP:conf/icml/MartinkusLPW22}   & 2.5 & 2.5 & 2.4 & \textbf{100.0}  & \textbf{100.0}  & 25.0\\
& ConGress~\cite{DBLP:conf/iclr/VignacKSWCF23}  & 23.8 & 8.8 & 2590.0 & 0.0 & 0.0 & 0.0\\
& DiGress~\cite{DBLP:conf/iclr/VignacKSWCF23}   & 1.4 & \textbf{1.2} & \textbf{1.7} & \textbf{100.0}  & \textbf{100.0}  & 85.0\\
& \cellcolor{Gray}\textsc{DisCo}-MPNN & \cellcolor{Gray}1.4\scriptsize$\pm$0.3 & \cellcolor{Gray}1.4\scriptsize$\pm$0.4 & \cellcolor{Gray}6.4\scriptsize$\pm$1.6 & \cellcolor{Gray}\textbf{100.0\scriptsize$\pm$0.0} & \cellcolor{Gray}\textbf{100.0\scriptsize$\pm$0.0} & \cellcolor{Gray}33.8\scriptsize$\pm$2.7 \\
& \cellcolor{Gray}\textsc{DisCo}-GT   & \cellcolor{Gray}\textbf{1.2}\scriptsize$\pm$0.5 & \cellcolor{Gray}1.3\scriptsize$\pm$0.5 & \cellcolor{Gray}\textbf{1.7\scriptsize$\pm$0.7} & \cellcolor{Gray}\textbf{100.0\scriptsize$\pm$0.0}  & \cellcolor{Gray}\textbf{100.0\scriptsize$\pm$0.0} & \cellcolor{Gray}83.6\scriptsize$\pm$2.1 \\\bottomrule
\end{tabular}
}
\end{table}

\begin{table}[t!]
\centering
\caption{Performance (mean$\pm$std\%) on QM9 dataset. V., U., and N. mean Valid, Unique, and Novel.}
\label{tab: effectiveness comparison on qm9}
\begin{tabular}{llll}
\toprule
Model       & Valid $\uparrow$ & V.U. $\uparrow$ & V.U.N. $\uparrow$ \\\midrule
CharacterVAE~\cite{DBLP:journals/corr/Gomez-Bombarelli16} & 10.3             & 7.0             & 6.3               \\
GrammarVAE\cite{DBLP:journals/corr/KusnerH16}   & 60.2             & 5.6             & 4.5               \\
GraphVAE~\cite{DBLP:conf/icann/SimonovskyK18}     & 55.7             & 42.0            & 26.1              \\
GT-VAE~\cite{DBLP:journals/corr/abs-2104-04345}       & 74.6             & 16.8            & 15.8              \\
Set2GraphVAE~\cite{DBLP:conf/iclr/VignacF22} & 59.9             & 56.2            & -                 \\
GG-GAN~\cite{krawczuk2020gg}       & 51.2             & 24.4            & 24.4              \\
MolGAN~\cite{DBLP:journals/corr/abs-1805-11973}       & 98.1             & 10.2            & 9.6               \\
SPECTRE~\cite{DBLP:conf/icml/MartinkusLPW22}      & 87.3             & 31.2            & 29.1              \\
GraphNVP~\cite{DBLP:journals/corr/abs-1905-11600}     & 83.1             & 82.4            & -                 \\
GDSS~\cite{DBLP:conf/icml/JoLH22}         & 95.7             & 94.3            & -                 \\
EDGE~\cite{chen2023efficient} & 99.1 & \textbf{99.1} & - \\
ConGress~\cite{DBLP:conf/iclr/VignacKSWCF23}     & 98.9             & 95.7            & 38.3              \\
DiGress~\cite{DBLP:conf/iclr/VignacKSWCF23}      & 99.0    & 95.2            & 31.8              \\
GRAPHARM~\cite{DBLP:conf/icml/KongCSZPZ23}     & 90.3             & 86.3            & -                 \\
\rowcolor{Gray}\textsc{DisCo}-MPNN    & 98.9\scriptsize$\pm$0.7 & 98.7\scriptsize$\pm$0.5 & \textbf{68.7\scriptsize$\pm$0.2}     \\
\rowcolor{Gray}\textsc{DisCo}-GT      & \textbf{99.3\scriptsize$\pm$0.6} & 98.9\scriptsize$\pm$0.6 & 56.2\scriptsize$\pm$0.4              \\
\bottomrule 
\end{tabular}
\end{table}

\subsection{Molecule Graph Generation}
\label{sec: molecule graph experiments}

\textbf{Dataset and metrics.}
The datasets QM9~\cite{ramakrishnan2014quantum}, MOSES~\cite{DBLP:journals/corr/abs-1811-12823}, and GuacaMol~\cite{DBLP:journals/jcisd/BrownFSV19} are chosen. For MOSES, metrics including Uniquess, Novelty, Validity, Filter, FCD, SNN, and Scaf are reported in Table~\ref{tab: effectiveness comparison on MOSES}. For QM9, metrics include Uniqueness, Novelty, and Validity. For GuacaMol, metrics include Valid, Unique, Novel, KL div, and FCD. Details about the datasets, metrics, and baseline methods are in Appendix~\ref{sec: detailed settings on molecule data}.


\begin{table}[t]
\centering
\caption{Performance on MOSES. VAE, JT-VAE, and GraphINVENT have hard-coded rules to ensure high validity.}
\label{tab: effectiveness comparison on MOSES}
\begin{tabular}{llllllll}
\toprule
Model & Valid $\uparrow$ & Unique $\uparrow$ & Novel $\uparrow$ & Filters $\uparrow$ & FCD $\downarrow$ & SNN $\uparrow$ & Scaf $\uparrow$ \\\midrule
VAE~\cite{gomez2018automatic} & 97.7 & 98.8 & 69.5 & 99.7 & 0.57 & 0.58 & 5.9 \\
JT-VAE~\cite{DBLP:conf/icml/JinBJ18} & 100.0 & 100.0 & 99.9 & 97.8 & 1.00 & 0.53 & 10.0 \\
GraphINVENT~\cite{mercado2021graph} & 96.4 & 99.8 & N/A & 95.0 & 1.22 & 0.54 & 12.7 \\
ConGress~\cite{DBLP:conf/iclr/VignacKSWCF23} & 83.4 & 99.9 & 96.4 & 94.8 & 1.48 & 0.50 & 16.4 \\
DiGress~\cite{DBLP:conf/iclr/VignacKSWCF23} & 85.7 & 100.0 & 95.0 & 97.1 & 1.19 & 0.52 & 14.8 \\
\rowcolor{Gray}\textsc{DisCo}-MPNN & 83.9 & 100.0 & 98.8 & 87.3 & 1.63 & 0.48 & 13.5 \\
\rowcolor{Gray}\textsc{DisCo}-GT & 88.3 & 100.0 & 97.7 & 95.6 & 1.44 & 0.50 & 15.1 \\\bottomrule
\end{tabular}
\end{table}

\begin{table}[t]
\centering
\caption{Performance on GuacaMol. LSTM, NAGVAE, and MCTS are tailored for molecule datasets; ConGress, DiGress, and \textsc{DisCo} are general graph generation models.}
\label{tab: effectiveness comparison on GuacaMol}
\begin{tabular}{llllll}
\toprule
Model & Valid $\uparrow$ & Unique $\uparrow$ & Novel $\uparrow$ & KL div $\uparrow$ & FCD $\uparrow$ \\\midrule
LSTM~\cite{segler2018generating} & 95.9 & 100.0 & 91.2 & 99.1 & 91.3 \\
NAGVAE~\cite{kwon2020compressed} & 92.9 & 95.5 & 100.0 & 38.4 & 0.9 \\
MCTS~\cite{jensen2019graph} & 100.0 & 100.0 & 95.4 & 82.2 & 1.5 \\
ConGress~\cite{DBLP:conf/iclr/VignacKSWCF23} & 0.1 & 100.0 & 100.0 & 36.1 & 0.0 \\
DiGress~\cite{DBLP:conf/iclr/VignacKSWCF23} & 85.2 & 100.0 & 99.9 & 92.9 & 68.0 \\
\rowcolor{Gray}\textsc{DisCo}-MPNN & 68.7 & 100.0 & 96.4 & 77.0 & 36.4 \\
\rowcolor{Gray}\textsc{DisCo}-GT & 86.6 & 100.0 & 99.9 & 92.6 & 59.7 \\
\bottomrule
\end{tabular}
\end{table}

\textbf{Results. }
Table~\ref{tab: effectiveness comparison on qm9} shows the performance on QM9 dataset. Our observation is consistent with the performance comparison on plain datasets: (1) \textsc{DisCo}-GT obtains slightly better or at least competitive performance against DiGress due to the shared graph-to-graph backbone, but our framework offers extra flexibility in the sampling process; (2) \textsc{DisCo}-MPNN obtains decent performance in terms of Validity, Uniqueness, and Novelty comparing with \textsc{DisCo}-GT.

Tables~\ref{tab: effectiveness comparison on MOSES} and~\ref{tab: effectiveness comparison on GuacaMol} show the performance on MOSES and GuacaMol which further verifies that (1) performance of \textsc{DisCo}-GT is on par with the SOTA general graph generative models, DiGress and (2) \textsc{DisCo}-MPNN has decent performance, but worse than \textsc{DisCo}-GT and DiGress.


\subsection{Efficiency Study}
\label{sec: efficiency study}

\begin{wraptable}{r}{0.45\textwidth}
\centering
\vspace{-5mm}
\caption{Efficiency comparison in terms of number of parameters, forward and backpropagation time (second/iteration).}
\label{tab: efficiency study}
\begin{tabular}{lll}
\toprule
         & GT         & MPNN      \\\midrule
\# Parameters & $14\times 10^6$ & $7\times 10^6$ \\
Forward        & $0.065$      & $0.022$     \\
Backprop.  & $0.034$      & $0.018$    
\\\bottomrule
\end{tabular}
\end{wraptable}

A major computation bottleneck is the graph-to-graph backbone $p_{0|t}^{\theta}$, which is GT or MPNN. We compare the number of parameters, the forward and back-propagation time of GT and MPNN in Table~\ref{tab: efficiency study}. For a fair comparison, we set all the hidden dimensions of GT and MPNN as $256$ and the number of layers as $5$. We use the Community~\cite{DBLP:conf/icml/YouYRHL18} dataset and set the batch size as $64$. Table~\ref{tab: efficiency study} shows that GT has a larger capacity and more parameters at the expense of more expensive training.


\begin{wraptable}{r}{0.61\textwidth}
\centering
\vspace{-7mm}
\caption{Ablation study (mean$\pm$std\%) with GT backbone. V., U., and N. mean Valid, Unique, and Novel.}
\label{tab: ablation study GT}
\begin{tabular}{lllll}
\toprule
Ref. Dist. & Steps & Valid $\uparrow$ & V.U. $\uparrow$ & V.U.N. $\uparrow$ \\\midrule
\multirow{6}{*}{Marginal} & 500 & 99.3\scriptsize$\pm$0.6 & 98.9\scriptsize$\pm$0.6 & 56.2\scriptsize$\pm$0.4 \\
& 100 & 98.7\scriptsize$\pm$0.5 & 98.5\scriptsize$\pm$0.4 & 58.8\scriptsize$\pm$0.4 \\
& 30 & 97.9\scriptsize$\pm$1.2 & 97.6\scriptsize$\pm$1.1 & 59.2\scriptsize$\pm$0.8 \\
& 10 & 95.3\scriptsize$\pm$1.9 & 94.8\scriptsize$\pm$1.6 & 62.1\scriptsize$\pm$0.9 \\
& 5 & 93.0\scriptsize$\pm$1.7 & 92.4\scriptsize$\pm$1.3 & 64.9\scriptsize$\pm$1.1 \\
& 1 & 76.1\scriptsize$\pm$2.3 & 73.9\scriptsize$\pm$1.6 & 62.9\scriptsize$\pm$1.8 \\\midrule
\multirow{6}{*}{Uniform} & 500 & 94.1\scriptsize$\pm$0.9 & 92.9\scriptsize$\pm$0.5 & 56.6\scriptsize$\pm$0.4 \\
& 100 & 91.5\scriptsize$\pm$1.0 & 90.3\scriptsize$\pm$0.9 & 54.4\scriptsize$\pm$1.2 \\
& 30 & 88.7\scriptsize$\pm$1.6 & 86.9\scriptsize$\pm$1.0 & 58.6\scriptsize$\pm$2.1 \\
& 10 & 84.5\scriptsize$\pm$2.3 & 80.4\scriptsize$\pm$1.7 & 59.8\scriptsize$\pm$1.8 \\
& 5 & 77.0\scriptsize$\pm$2.5 & 69.9\scriptsize$\pm$1.5 & 56.1\scriptsize$\pm$3.5 \\
& 1 & 44.9\scriptsize$\pm$3.1 & 35.1\scriptsize$\pm$3.4 & 29.6\scriptsize$\pm$2.5 \\\bottomrule
\end{tabular}
\end{wraptable}

\eject
\subsection{Ablation Study}
\label{sec: ablation study}

An ablation study on \textsc{DisCo}-GT for reference distributions (marginal vs. uniform), and sampling steps ($1$ to $500$) is presented in Table~\ref{tab: ablation study GT}. The number of sampling steps is $\texttt{round}(\frac{1}{\tau})$ if $T=1$. QM9 dataset is chosen. A similar ablation study on \textsc{DisCo}-MPNN is in Table~\ref{tab: ablation study MPNN} in Appendix. We observe that first, generally, the fewer sampling steps, the lower the generation quality. In some cases (e.g., the marginal distribution) with the sampling steps decreasing significantly (e.g., from $500$ to $30$), the performance degradation is still very slight, implying our method's high robustness in sampling steps. Second, the marginal reference distribution is better than the uniform distribution, consistent with the observation from DiGress~\cite{DBLP:conf/iclr/VignacKSWCF23}.

\section{Related Work}
Diffusion models~\cite{DBLP:journals/corr/abs-2209-00796} can be interpreted from both the score-matching~\cite{DBLP:conf/nips/SongE19,DBLP:conf/iclr/0011SKKEP21} or the variational autoencoder perspective~\cite{DBLP:conf/nips/HoJA20,DBLP:journals/corr/abs-2107-00630,DBLP:conf/nips/KingmaG23}. Pioneering efforts on diffusion generative modeling study the process in continuous-state~\cite{DBLP:conf/icml/Sohl-DicksteinW15, DBLP:conf/nips/HoJA20, DBLP:conf/iclr/SongME21} whose typical reference distribution is Gaussian. Beyond that, some efforts propose discrete-state models~\cite{DBLP:conf/nips/HoogeboomNJFW21} to . E.g., D3PM~\cite{DBLP:conf/nips/AustinJHTB21} designs the discrete diffusion process by multiplication of transition matrices; $\tau$-LDR~\cite{DBLP:conf/nips/CampbellBBRDD22} generalizes D3PM by formulating a continuous-time Markov chain; ~\cite{DBLP:conf/iclr/SunYDSD23} proposes a singleton conditional distribution-based objective for the continuous-time Markov chain-based model whose rationale is Brook's Lemma~\cite{brook1964distinction,DBLP:conf/uai/Lyu09}.

Diffusion models are widely used in graph generation tasks~\cite{DBLP:conf/ijcai/LiuFLLLLTL23, DBLP:journals/sigkdd/DingXTL22, DBLP:conf/kdd/ZhouZ0H20,DBLP:journals/fdata/ZhouZXH19,DBLP:conf/icde/Zheng0TXZH24,DBLP:conf/wsdm/FuXTH23,DBLP:conf/icml/ZengQ00YW0HT24,DBLP:journals/fdata/FuH22,DBLP:conf/kdd/QiuWYPZT23,DBLP:conf/kdd/WangYQZGMT23,DBLP:conf/icml/0002QZYZ0ZWHT24,DBLP:conf/www/XuDT22,DBLP:conf/kdd/XuCPCDYT23,DBLP:conf/kdd/QiBH23,DBLP:conf/www/BanQH24} such as molecule design~\cite{DBLP:conf/icml/ShiLXT21,DBLP:conf/icml/HoogeboomSVW22,DBLP:conf/aaai/HuangZXW23,DBLP:journals/corr/abs-2211-11214}. Pioneering works such as EDP-GNN~\cite{DBLP:conf/aistats/NiuSSZGE20} and GDSS~\cite{DBLP:conf/icml/JoLH22} diffuse graph data in a continuous state space~\cite{DBLP:conf/icdm/HuangS0FL22}. DiscDDPM~\cite{DBLP:journals/corr/abs-2210-01549} is an early effort to modify the DDPM architecture into a discrete state. In addition, DiGress~\cite{DBLP:conf/iclr/VignacKSWCF23} is also a one-shot discrete-state diffusion model, followed by a very recent work MCD~\cite{liu2024inverse}, both in the discrete-time setting. Beyond the above-mentioned efforts, DruM~\cite{jo2023graph} proposes to mix the diffusion process. EDGE~\cite{chen2023efficient} proposes an interesting process: diffusing graphs into empty graphs. Besides, GRAPHARM~\cite{DBLP:conf/icml/KongCSZPZ23} proposes an autoregressive graph diffusion model, and ~\cite{liu2024data} applies the diffusion models for molecule property prediction tasks. In addition to the above-mentioned general graph diffusion models, there are many other task-tailored graph diffusion generative models~\cite{DBLP:conf/nips/LuoSXT21,DBLP:conf/nips/JingCCBJ22,DBLP:conf/icml/LeeJH23,DBLP:journals/corr/abs-2209-15171,DBLP:conf/iclr/XuY0SE022, xu2023geometric,bao2022equivariant,weiss2023guided}, which incorporate more in-depth domain expertise into the model design. Interested readers are referred to this survey~\cite{DBLP:conf/ijcai/LiuFLLLLTL23}.
\section{Conclusion}

This paper introduces the first discrete-state continuous-time graph diffusion generative model, \textsc{DisCo}. Our model effectively marries continuous-time Markov Chain formulation with the discrete nature of graph data, addressing the fundamental sampling limitation of prior models. \textsc{DisCo}'s training objective is concise with a solid theoretical foundation. We also propose a simplified message-passing architecture to serve as the graph-to-graph backbone, which theoretically has desirable properties against permutation of node ordering and empirically demonstrates decent performance against existing graph generative models in tests on various datasets.

\clearpage

\begin{ack}
\thanks{ZX, RQ, ZZ, and HT are partially supported by NSF (2324770). 
The content of the information in this document does not necessarily reflect the position or the policy of the Government, and no official endorsement should be inferred.  The U.S. Government is authorized to reproduce and distribute reprints for Government purposes notwithstanding any copyright notation here on.
}
\end{ack}

\bibliography{ref}
\bibliographystyle{plain}


\vfill\eject
\appendix

\textbf{\huge Appendix}

\vspace{3mm}
The organization of this appendix is as follows
\begin{itemize}
    \item Section~\ref{sec: rate factorization appendix}: the derivation of the factorized rate matrices.
    \item Section~\ref{sec: forward appendix}: the forward transition probability matrix given a rate matrix.
    \item Section~\ref{sec: proofs appendix}: all the proofs.
    \begin{itemize}
        \item Section~\ref{sec: proof of rate matrices}: proof of Proposition~\ref{prop: rate matrices for uniform and marginal dist}
        \item Section~\ref{sec: proof of objective function}: proof of Theorem~\ref{thm: correctness of objective function}
        \item Section~\ref{sec: proof of permutation layer}: proof of Lemma~\ref{lemma: permutation equivalent layer}
        \item Section~\ref{sec: proof of permutation-invariant rate}: proof of Lemma~\ref{lemma: permutation-invariant rate}
        \item Section~\ref{sec: proof of permutation-invariant transition}: proof of Lemma~\ref{lemma: permutation-invariant transition}
        \item Section~\ref{sec: proof of permutation density}: proof of Theorem~\ref{thm: permutation-invariant density}
        \item Section~\ref{sec: proof of permutation loss}: proof of Theorem~\ref{thm: permutation-invariant training loss}
    \end{itemize}
    \item Section~\ref{sec: algorithms appendix}: a step-by-step sampling algorithm.
    \item Section~\ref{sec: auxiliary features}: the auxiliary features and neural modules used by \textsc{DisCo}.
    \item Section~\ref{sec: experiments appendix}: detailed experimental settings and additional experimental results.
    \begin{itemize}
        \item Section~\ref{sec: hardware and implmentations}: hardware and software
        \item Section~\ref{sec: dataset settings appendix}: dataset setup
        \item Section~\ref{sec: hyperparameter settings}: hyperparameter settings
        \item Section~\ref{sec: plain graph results appendix}: additional effectiveness evaluation on Community Dataset
        \item Section~\ref{sec: ablation study appendix}: additional ablation study with the MPNN backbone
        \item Section~\ref{sec: convergence study}: convergence study
        \item Section~\ref{sec: visualization}: visualization
    \end{itemize}
    \item Section~\ref{sec: limitation}: this paper's limitations and future work.
    \item Section~\ref{sec: broader impact}: the broad impact of this paper.
\end{itemize}

\section{Details of the Factorization of Rate Matrices}
\label{sec: rate factorization appendix}
In this section, we detail the derivation of Remark~\ref{prop: rate factorization}, which is extended from the following Proposition 3 of~\cite{DBLP:conf/nips/CampbellBBRDD22}.
\begin{proposition}[Factorization of the rate matrix, Proposition 3 from~\cite{DBLP:conf/nips/CampbellBBRDD22}]If the forward process factorizes as $q_{t|s}(\mathbf x_t|\mathbf x_s)=\prod_{d=1}^D q_{t|s}(x_t^d|x_s^d), t>s$, then the forward and reverse rates are of the form
\begin{align}
    \mathbf R_t(\bar{\mathbf x}, \mathbf x) &= \sum_{d=1}^D \mathbf R_t^d (\bar{x}^d, x^d)\delta_{\bar{\mathbf{x}}\setminus \bar{x}^d, \mathbf{x}\setminus x^d}\\
    \tilde{\mathbf R}_t(\mathbf x, \bar{\mathbf x}) &= \sum_{d=1}^D \mathbf R_t^d (\bar{x}^d, x^d)\delta_{\bar{\mathbf{x}}\setminus \bar{x}^d, \mathbf{x}\setminus x^d} \sum_{x_0^d}q_{0|t}(x_0^d|\mathbf x)\frac{q_{t|0}(\bar{x}^d|x_0^d)}{q_{t|0}(x^d|x_0^d)}
\end{align}
where $\delta_{\bar{\mathbf{x}}\setminus \bar{x}^d, \mathbf{x}\setminus x^d}=1$ when all dimensions except for $d$ are equal.
\end{proposition}

As all the nodes and edges are categorical, applying the above proposition of all the nodes and edges leads to our Remark~\ref{prop: rate factorization}.

\section{Details of Forward Transition Probability}
\label{sec: forward appendix}



In this section, we present the derivation of the forward transition probability for nodes; the forward process for edges can be derived similarly. Note that this derivation has been mentioned in ~\cite{DBLP:conf/nips/CampbellBBRDD22} for generic discrete cases; we graft it to the graph settings and include it here for completeness. The core derivation of the forward transition probability is to prove the following proposition.

\begin{proposition}[Analytical forward process for commutable rate matrices, Proposition 10 from~\cite{DBLP:conf/nips/CampbellBBRDD22}]
    \label{prop: analytical forward process}
    if $\mathbf{R}_t$ and $\mathbf{R}_{t'}$ commute $\forall t, t'$, $q_{t|0}(x_t=j|x_0=i)=(e^{\int_{0}^t\mathbf{R}_s ds})_{ij}$
\end{proposition}

\begin{proof}
    If $q_{t|0}=\exp\bigg(\int_{0}^t\mathbf{R}_s ds\bigg)$ is the forward transition probability matrix, it should satisfy the Kolmogorov forward equation $\frac{d}{dt}q_{t|0} = q_{t|0}\mathbf{R}_s$. The transition probability matrix
\begin{align}
    q_{t|0} = \sum_{k=0}^{\infty}\frac{1}{k!}\bigg(\int_{0}^t\mathbf{R}_s ds\bigg)^k,
\end{align}
and, based on the fact that $\mathbf{R}_t$ and $\mathbf{R}_t'$ commute $\forall t,t'$ , its derivative is
\begin{align}
    \frac{d}{dt}q_{t|0} = \sum_{k=1}^{\infty}\frac{1}{(k-1)!}\bigg(\int_{0}^t\mathbf{R}_sds\bigg)^{(k-1)} = q_{t|0}\mathbf{R}_t.
\end{align}
Thus, $q_{t|0}=\exp\bigg(\int_{0}^t\mathbf{R}_sds\bigg)$ is the solution of Kolmogorov forward equation.
\end{proof}

For the node $i$, if its forward rate matrix is set as $\mathbf{R}_t^i=\beta(t)\mathbf{R}_f$, we have $\mathbf{R}^i_t$ and $\mathbf{R}^i_{t'}$ commute, $\forall t,t'$. Thus, the transition probability for node $i$ is $q_{t|0}(f^i_t=v|f^i_0=u)=(e^{\int_{0}^t\beta(s)\mathbf{R}_f ds})_{uv}$. Based on similar derivation, we have the transition probability for the edge $(i,j)$ as $q_{t|0}(e^{(i,j)}_t=v|e^{(i,j)}_0=u)=(e^{\int_{0}^t\beta(s)\mathbf{R}_e ds})_{uv}$.

\section{Proofs}
\label{sec: proofs appendix}

\subsection{Proof of Proposition~\ref{prop: rate matrices for uniform and marginal dist}}
\label{sec: proof of rate matrices}
Proposition~\ref{prop: rate matrices for uniform and marginal dist} claims the forward process converges to uniform distributions if $\mathbf{R}_f=\mathbf{1}\mathbf{1}^{\top}-b\mathbf{I}$ and $\mathbf{R}_e=\mathbf{1}\mathbf{1}^{\top}-(a+1)\mathbf{I}$ and it converges to marginal distributions $\mathbf{m}_f$ and $\mathbf{m}_e$ if $\mathbf{R}_f=\mathbf{1}\mathbf{m}_f^{\top}-\mathbf{I}$ and $\mathbf{R}_e=\mathbf{1}\mathbf{m}_e^{\top}-\mathbf{I}$.

\begin{proof}

If we formulate the rate matrices for nodes and edges as $\mathbf{R}^{(i,j)}_t=\beta(t)\mathbf{R}_e,\ \forall i,j$ and $\mathbf{R}^{i}_t=\beta(t)\mathbf{R}_f,\ \forall i$, every rate matrix is commutable for any time steps $t$ and $t'$. In the following content, we show the proof for the node rate matrix $\mathbf{R}^{i}_t=\beta(t)\mathbf{R}_f$; the converged distribution of edge can be proved similarly. Based on Proposition~\ref{prop: analytical forward process}, the transition probability matrix between time steps $t$ and $t+\Delta t$ is
\begin{align}
    q_{t+\Delta t|t} &= \mathbf{I} + \int_{t}^{t+\Delta t} \beta(s)\mathbf{R}_f ds + O((\Delta t)^2)\\
    &\overset{(*)}{=} \mathbf{I} + \Delta t \beta(\xi) \mathbf{R}_f +  O((\Delta t)^2),
\end{align}
where (*) is based on the Mean Value Theorem. If the high-order term $O((\Delta t)^2)$ is omitted and we short $\beta_{\Delta t} = \Delta t \beta(\xi)$, for $\mathbf{R}_f=\mathbf{1}\mathbf{1}^{\top}-b\mathbf{I}$, we have
\begin{align}
    q_{t+\Delta t|t} \approx \beta_{\Delta t}\mathbf{1}\mathbf{1}^{\top} + (1-\beta_{\Delta t}b)\mathbf{I},
\end{align}
which is the transition matrix of the uniform diffusion in the discrete-time diffusion models~\cite{DBLP:conf/icml/Sohl-DicksteinW15,DBLP:conf/nips/AustinJHTB21}. Thus, with $T\rightarrow \infty$ and $q_{t+\Delta t|t}$ to the power of infinite, the converged distribution is a uniform distribution. Similarly, for $\mathbf{R}_f=\mathbf{1}\mathbf{m}_f^{\top}-\mathbf{I}$ the transition matrix is
\begin{align}
    q_{t+\Delta t|t} \approx \beta_{\Delta t}\mathbf{1}\mathbf{m}_f^{\top} + (1-\beta_{\Delta t})\mathbf{I}
\end{align}
which is a generalized transition matrix of the `absorbing state' diffusion~\cite{DBLP:conf/nips/AustinJHTB21}. The difference lies at for the `absorbing state' diffusion~\cite{DBLP:conf/nips/AustinJHTB21}, $\mathbf{m}_f$ is set as a one-hot vector for the absorbing state, and here we set it as the marginal distribution. Thus, with $T\rightarrow \infty$ and $q_{t+\Delta t|t}$ to the power of infinite, the converged distribution is a marginal distribution $\mathbf{m}_f$.
\end{proof}

\subsection{Proof of Theorem~\ref{thm: correctness of objective function}}
\label{sec: proof of objective function}
Theorem~\ref{thm: correctness of objective function} says for $\mathcal{G} \neq \bar{\mathcal{G}}$,

\begin{align}
        \Big| \tilde{\mathbf{R}}_{t}(\mathcal{G},\bar{\mathcal{G}})-\tilde{\mathbf{R}}_{\theta,t}(\mathcal{G},\bar{\mathcal{G}}) \Big|^2 &\leq C_t + C^{\texttt{node}}_t \mathbb E_{\mathcal G_0} q_{t|0}(\mathcal G|\mathcal G_0)\sum_i\mathcal{L}_{\mathtt{CE}}(\mathtt{One\text-Hot}(f_0^i), \hat{f}_0^i)\nonumber \\ 
        &+ C^{\texttt{edge}}_t\mathbb E_{\mathcal G_0} q_{t|0}(\mathcal G|\mathcal G_0)\sum_{i,j}\mathcal{L}_{\mathtt{CE}}(\mathtt{One\text-Hot}(e_0^{(i,j)}), \hat{e}_0^{(i,j)})
    \end{align}
    where the node and edge estimated probability vector (sum is $1$) is notated as $\hat{f}^i_0 = [p^{\theta}_{0|t}(f^i=1 | \mathcal{G}_t),\dots, p^{\theta}_{0|t}(f^i=b | \mathcal{G}_t)]^\top\in[0,1]^b$ and $\hat{e}^{(i,j)}_0=[p^{\theta}_{0|t}(e^{(i,j)}=1| \mathcal{G}_t), \dots, p^{\theta}_{0|t}(e^{(i,j)}=a+1| \mathcal{G}_t)]^\top \in[0,1]^{a+1}$.

\begin{proof}
    \begin{align}
        &  \Big| \tilde{\mathbf{R}}_{t}(\mathcal{G},\bar{\mathcal{G}})-\tilde{\mathbf{R}}_{\theta,t}(\mathcal{G},\bar{\mathcal{G}}) \Big|\\
        =&  \Big| \sum_i A_t^i \sum_{f_0^i}\frac{q_{t|0}(\bar{f}^i|f^i_0)}{q_{t|0}(f^i|f^i_0)} (q_{0|t}(f^i_0|\mathcal{G})- p^{\theta}_{0|t}(f^i_0|\mathcal{G})) \nonumber\\ &+ \sum_{i,j} B_t^{(i,j)} \sum_{e_0^{(i,j)}} \frac{q_{t|0}(\bar{e}^{(i,j)}|e^{(i,j)}_0)}{q_{t|0}(e^{(i,j)}|e^{(i,j)}_0)} (q_{0|t}(e^{(i,j)}_0|\mathcal{G})- p^{\theta}_{0|t}(e^{(i,j)}_0|\mathcal{G})) \Big|\\
        \leq& \Big| \sum_i A_t^i \sum_{f_0^i}\frac{q_{t|0}(\bar{f}^i|f^i_0)}{q_{t|0}(f^i|f^i_0)} (q_{0|t}(f^i_0|\mathcal{G})- p^{\theta}_{0|t}(f^i_0|\mathcal{G}))\Big| \nonumber\\ &+ \Big|\sum_{i,j} B_t^{(i,j)} \sum_{e_0^{(i,j)}} \frac{q_{t|0}(\bar{e}^{(i,j)}|e^{(i,j)}_0)}{q_{t|0}(e^{(i,j)}|e^{(i,j)}_0)} (q_{0|t}(e^{(i,j)}_0|\mathcal{G})- p^{\theta}_{0|t}(e^{(i,j)}_0|\mathcal{G})) \Big|\label{eq: total variance}
    \end{align}
    We check the first term of Eq. \eqref{eq: total variance}:
    \begin{align}
        &\Big| \sum_i A_t^i \sum_{f_0^i}\frac{q_{t|0}(\bar{f}^i|f^i_0)}{q_{t|0}(f^i|f^i_0)} (q_{0|t}(f^i_0|\mathcal{G})- p^{\theta}_{0|t}(f^i_0|\mathcal{G}))\Big|\\
        \leq& \sum_i A_t^i \sup_{f_0^i}\Big\{\frac{q_{t|0}(\bar{f}^i|f^i_0)}{q_{t|0}(f^i|f^i_0)}\Big\}\sum_{f_0^i}\Big|q_{0|t}(f^i_0|\mathcal{G})- p^{\theta}_{0|t}(f^i_0|\mathcal{G})\Big|\\
        =& \sum_i C_i \sum_{f_0^i}\Big|q_{0|t}(f^i_0|\mathcal{G})- p^{\theta}_{0|t}(f^i_0|\mathcal{G})\Big|\\
        \overset{(*)}{\leq}& \sum_i C_i\sqrt{2\sum_{f_0^i}\Big(C_{f^i_0}-q_{0|t}(f^i_0|\mathcal{G})\log p^{\theta}_{0|t}(f^i_0|\mathcal{G})\Big)}\\
        \overset{(**)}{\leq}& C_1 \sqrt{\sum_i\sum_{f_0^i}\Big(C_{f^i_0}-q_{0|t}(f^i_0|\mathcal{G})\log p^{\theta}_{0|t}(f^i_0|\mathcal{G})\Big)}\\
        =& C_1 \sqrt{C_2 - \sum_i\sum_{f_0^i}q_{0|t}(f^i_0|\mathcal{G})\log p^{\theta}_{0|t}(f^i_0|\mathcal{G})}\label{eq: qp bound on nodes}
    \end{align}
    where $C_i = A_t^i \sup_{f_0^i}\Big\{\frac{q_{t|0}(\bar{f}^i|f^i_0)}{q_{t|0}(f^i|f^i_0)}\Big\}$,  $C_{f^i_0}=q_{0|t}(f^i_0|\mathcal{G})\log q_{0|t}(f^i_0|\mathcal{G})$, (*) is based on the Pinsker's inequality, (**) is based on Cauchy–Schwarz inequality: $\sum_{i=1}^{n}\sqrt{x_i} \leq \sqrt{n\sum_{i=1}^n x_i}$, $C_1=\sqrt{2n}\sup_i\{C_i\}$, $C_2=\sum_i\sum_{f_0^i}C_{f^i_0}$. Next, the term $- \sum_i\sum_{f_0^i}q_{0|t}(f^i_0|\mathcal{G})\log p^{\theta}_{0|t}(f^i_0|\mathcal{G})$ is equivalent to:
    \begin{align}
        &- \sum_i\sum_{f_0^i}q_{0|t}(f^i_0|\mathcal{G})\log p^{\theta}_{0|t}(f^i_0|\mathcal{G})\\
        =& - \frac{1}{p_t(\mathcal{G})}\sum_i\sum_{f_0^i}p_{0,t}(f^i_0,\mathcal{G})\log p^{\theta}_{0|t}(f^i_0|\mathcal{G})\\
        =& - \frac{1}{p_t(\mathcal{G})}\sum_i \sum_{f_0^i} \sum_{\mathcal G_0(f_0^i)} p_{0,t}(\mathcal G_0,\mathcal{G})\log p^{\theta}_{0|t}(f^i_0|\mathcal{G})\\
        =& - \frac{1}{p_t(\mathcal{G})}\sum_i \sum_{f_0^i} \sum_{\mathcal G_0(f_0^i)} \pi_{\texttt{data}}(\mathcal G_0)q_{t|0}(\mathcal G|\mathcal G_0)\log p^{\theta}_{0|t}(f^i_0|\mathcal{G})\\
        =& \frac{1}{p_t(\mathcal{G})}\sum_i \sum_{f_0^i} \sum_{\mathcal G_0(f_0^i)} \pi_{\texttt{data}}(\mathcal G_0)q_{t|0}(\mathcal G|\mathcal G_0)\mathcal{L}_{\mathtt{CE}}(\mathtt{One\text-Hot}(f_0^i), \hat{f}_0^i)\\
        =& \frac{1}{p_t(\mathcal{G})} \sum_{\mathcal G_0} \pi_{\texttt{data}}(\mathcal G_0)q_{t|0}(\mathcal G|\mathcal G_0) \sum_i\mathcal{L}_{\mathtt{CE}}(\mathtt{One\text-Hot}(f_0^i), \hat{f}_0^i)\\
        =& \frac{1}{p_t(\mathcal{G})} \mathbb E_{\mathcal G_0} q_{t|0}(\mathcal G|\mathcal G_0) \sum_i\mathcal{L}_{\mathtt{CE}}(\mathtt{One\text-Hot}(f_0^i), \hat{f}_0^i) \label{eq: ce bound on nodes}
    \end{align}
    where $\sum_{\mathcal G_0(f_0^i)}$ marginalizing all the graphs at time $0$ whose $i$-th node is $f_0^i$; $p_{0,t}(f^i_0,\mathcal{G})$ is the joint probability of a graph whose $i$-th node is $f_0^i$ at time $0$ and it is $\mathcal G$ at time $t$; $p_{0,t}(\mathcal G_0,\mathcal{G})$ is the joint probability of a graph which is $\mathcal G_0$ at time $0$ and it is $\mathcal G$ at time $t$. Plugging Eq.~\eqref{eq: ce bound on nodes} into Eq.~\eqref{eq: qp bound on nodes}:
    \begin{align}
       &\Big| \sum_i A_t^i \sum_{f_0^i}\frac{q_{t|0}(\bar{f}^i|f^i_0)}{q_{t|0}(f^i|f^i_0)} (q_{0|t}(f^i_0|\mathcal{G})- p^{\theta}_{0|t}(f^i_0|\mathcal{G}))\Big|\nonumber\\  \leq& C_1 \sqrt{C_2 + C_5 \mathbb E_{\mathcal G_0} q_{t|0}(\mathcal G|\mathcal G_0) \sum_i\mathcal{L}_{\mathtt{CE}}(\mathtt{One\text-Hot}(f_0^i), \hat{f}_0^i)}\label{eq: ce node bound final}
    \end{align}
    where $C_5 = \frac{1}{p_t(\mathcal G)}$. A similar analysis can be conducted about the second term of Eq. \eqref{eq: total variance} and we directly present it here:
    \begin{align}
       &\Big|\sum_{i,j} B_t^{(i,j)} \sum_{e_0^{(i,j)}} \frac{q_{t|0}(\bar{e}^{(i,j)}|e^{(i,j)}_0)}{q_{t|0}(e^{(i,j)}|e^{(i,j)}_0)} (q_{0|t}(e^{(i,j)}_0|\mathcal{G})- p^{\theta}_{0|t}(e^{(i,j)}_0|\mathcal{G})) \Big|\nonumber\\  \leq& C_3 \sqrt{C_4 + C_5 \mathbb E_{\mathcal G_0} q_{t|0}(\mathcal G|\mathcal G_0) \sum_{i,j}\mathcal{L}_{\mathtt{CE}}(\mathtt{One\text-Hot}(e_0^{(i,j)}), \hat{e}_0^{(i,j)})}\label{eq: ce edge bound final}
    \end{align}
    where $C_3=\sqrt{2}n\sup_{i,j}\{C_{i,j}\}$, $C_4=\sum_{i,j}\sum_{e_0^{(i,j)}}C_{e_0^{(i,j)}}$, $C_{i,j} = B_t^{(i,j)} \sup_{e_0^{(i,j)}}\Big\{\frac{q_{t|0}(\bar{e}^{(i,j)}|e^{(i,j)}_0)}{q_{t|0}(e^{(i,j)}|e^{(i,j)}_0)}\Big\}$,  $C_{e^{(i,j)}_0}=q_{0|t}(e^{(i,j)}_0|\mathcal{G})\log q_{0|t}(e^{(i,j)}_0|\mathcal{G})$.

    Plugging Eqs.~\eqref{eq: ce node bound final} and~\eqref{eq: ce edge bound final} into Eq.~\eqref{eq: total variance}, being aware that $C_1$, $C_2$, $C_3$, $C_4$, $C_5$ are all $t$-related:
    \begin{align}
        \Big| \tilde{\mathbf{R}}_{t}(\mathcal{G},\bar{\mathcal{G}})-\tilde{\mathbf{R}}_{\theta,t}(\mathcal{G},\bar{\mathcal{G}}) \Big| &\leq C_1 \sqrt{C_2 + C_5 \mathbb E_{\mathcal G_0} q_{t|0}(\mathcal G|\mathcal G_0) \sum_i\mathcal{L}_{\mathtt{CE}}(\mathtt{One\text-Hot}(f_0^i), \hat{f}_0^i)}\nonumber\\ 
        &+ C_3 \sqrt{C_4 + C_5 \mathbb E_{\mathcal G_0} q_{t|0}(\mathcal G|\mathcal G_0) \sum_{i,j}\mathcal{L}_{\mathtt{CE}}(\mathtt{One\text-Hot}(e_0^{(i,j)}), \hat{e}_0^{(i,j)})}\\
        &\overset{(*)}{\leq} \Big( C_t + C^{\texttt{node}}_t \mathbb E_{\mathcal G_0} q_{t|0}(\mathcal G|\mathcal G_0)\sum_i\mathcal{L}_{\mathtt{CE}}(\mathtt{One\text-Hot}(f_0^i), \hat{f}_0^i)\nonumber \\ 
        &+ C^{\texttt{edge}}_t\mathbb E_{\mathcal G_0} q_{t|0}(\mathcal G|\mathcal G_0)\sum_{i,j}\mathcal{L}_{\mathtt{CE}}(\mathtt{One\text-Hot}(e_0^{(i,j)}), \hat{e}_0^{(i,j)}) \Big)^{1/2} 
    \end{align}
    where (*) is based on Cauchy–Schwarz inequality, $C_t = 2C_1^2C_2+2C_3^2C_4$, $C_t^{\texttt{node}} = 2C_1^2C_5$, $C_t^{\texttt{edge}} = 2C_3^2C_5$.
\end{proof}

\subsection{Proof of Lemma~\ref{lemma: permutation equivalent layer}}
\label{sec: proof of permutation layer}
We clarify that the term "permutation" in this paper refers to the reordering of the node indices, i.e., the first dimension of $\mathbf{F}$ and the first two dimensions of $\mathbf{E}$.
\begin{proof}
The input of an MPNN layer is $\mathbf{F}=\{\mathbf{r}_i\}_{i=1}^n\in\mathbb{R}^{n\times d},\mathbf{E}=\{\mathbf{r}_{i,j}\}_{i,j=1}^n\in\mathbb{R}^{n\times n\times d},\mathbf{y}\in\mathbb{R}^{d}$, where $d$ is the hidden dimension. The updating formulas of an MPNN layer can be presented as
\begin{align}
    \mathbf{r}^i &\gets \mathtt{FiLM}\Bigg(\mathtt{FiLM}\Bigg(\mathbf{r}^i, \mathtt{MLP}\bigg(\frac{\sum_{j=1}^n \mathbf{r}^{(j,i)}}{n}\bigg)\Bigg), \mathbf y\Bigg)\label{eq: update node repeat},\\
    \mathbf{r}^{(i,j)} &\gets \mathtt{FiLM} \big( \mathtt{FiLM} (\mathbf{r}^{(i,j)}, \mathbf{r}^{i} \odot \mathbf{r}^{j}), \mathbf{y}\big)\label{eq: update edge repeat},\\
    \mathbf{y} &\gets \mathbf{y} + \mathtt{PNA}\big(\{\mathbf{r}^{i}\}_{i=1}^{n}\big) + \mathtt{PNA}\big(\{\mathbf{r}^{(i,j)}\}_{i,j=1}^{n}\big)\label{eq: update global repeat},
\end{align}
The permutation $\mathcal{P}$ of the input of an MPNN layer can be presented as $\mathcal{P}\big(\mathbf{F}=\{\mathbf{r}_i\}_{i=1}^n,\mathbf{E}=\{\mathbf{r}_{i,j}\}_{i,j=1}^n,\mathbf{y}\big)=\big(\{\mathbf{r}_{\sigma(i)}\}_{i=1}^n,\{\mathbf{r}_{\sigma(i),\sigma(j)}\}_{i,j=1}^n,\mathbf{y}\big)$ where $\sigma:\{1,\dots,n\}\mapsto \{1,\dots,n\}$ is a bijection.

For $\mathtt{PNA}$ (Eq.~\eqref{eq: PNA}), it includes operations $\mathtt{max}$, $\mathtt{min}$, $\mathtt{mean}$, and $\mathtt{std}$ which are all permutation-invariant and thus, the $\mathtt{PNA}$ module is permutation-invariant. Then,
\begin{align}
    \mathbf{y} + \mathtt{PNA}\big(\{\mathbf{r}^{i}\}_{i=1}^{n}\big) + \mathtt{PNA}\big(\{\mathbf{r}^{(i,j)}\}_{i,j=1}^{n}\big) = \mathbf{y} + \mathtt{PNA}\big(\{\mathbf{r}^{\sigma(i)}\}_{i=1}^{n}\big) + \mathtt{PNA}\big(\{\mathbf{r}^{(\sigma(i),\sigma(j))}\}_{i,j=1}^{n}\big)
\end{align}


Because $\sum_{j=1}^n \mathbf{r}^{(j,i)}=\sum_{j=1}^n \mathbf{r}^{(\sigma(j),\sigma(i))}$, $\mathbf{r}^{i} \odot \mathbf{r}^{j}=\mathbf{r}^{\sigma(i)} \odot \mathbf{r}^{\sigma(j)}$, and the FiLM module (Eq.~\eqref{eq: film}) is not related to the node ordering,
\begin{align}
    \mathbf r^{(\sigma(i),\sigma(j))}\gets \mathtt{FiLM} \big( \mathtt{FiLM} (\mathbf{r}^{(\sigma(i),\sigma(j))}, \mathbf{r}^{\sigma(i)} \odot \mathbf{r}^{\sigma(j)}), \mathbf{y}\big) = \mathtt{FiLM} \big( \mathtt{FiLM} (\mathbf{r}^{(i,j)}, \mathbf{r}^{i} \odot \mathbf{r}^{j}), \mathbf{y}\big)
\end{align}
\begin{align}
    \mathbf r^{\sigma(i)}&\gets \mathtt{FiLM}\Bigg(\mathtt{FiLM}\Bigg(\mathbf{r}^{\sigma(i)}, \mathtt{MLP}\bigg(\frac{\sum_{j=1}^n \mathbf{r}^{(\sigma(j),\sigma(i))}}{n}\bigg)\Bigg), \mathbf y\Bigg)\\ &= \mathtt{FiLM}\Bigg(\mathtt{FiLM}\Bigg(\mathbf{r}^i, \mathtt{MLP}\bigg(\frac{\sum_{j=1}^n \mathbf{r}^{(j,i)}}{n}\bigg)\Bigg), \mathbf y\Bigg)
\end{align}
Thus, we proved that 
\begin{align}
    \mathtt{MPNN}\bigg(\mathcal{P}\big(\mathbf{F},\mathbf{E},\mathbf{y}\big)\bigg) = \mathcal{P}\bigg(\mathtt{MPNN}(\mathbf{F},\mathbf{E},\mathbf{y})\bigg)
\end{align}
\end{proof}

\subsection{Proof of Lemma~\ref{lemma: permutation-invariant rate}}
\label{sec: proof of permutation-invariant rate}
\begin{proof}
    The forward rate matrix (Eq.~\eqref{eq: forward rate matrix}) is the sum of component-specific forward rate matrices ($\{\mathbf{R}^{(i,j)}_t\}_{i,j\in\mathbb{N}^{+}_{\leq n}}$ and $\{\mathbf{R}^i_t\}_{i\in\mathbb{N}^{+}_{\leq n}}$). It is permutation-invariant because the summation is permutation-invariant.

    The parametric reverse rate matrix is \begin{align}
        \tilde{\mathbf{R}}_{\theta,t} (\mathcal{G}, \bar{\mathcal{G}}) = \sum_i \tilde{\mathbf{R}}^i_{\theta,t}(f^i,\bar{f}^i) + \sum_{i,j} \tilde{\mathbf{R}}^{(i,j)}_{\theta,t}(e^{(i,j)},\bar{e}^{(i,j)})
    \end{align}
    where $\tilde{\mathbf{R}}^i_{\theta,t}(f^i,\bar{f}^i) = A_t^i \sum_{f_0^i}\frac{q_{t|0}(\bar{f}^i|f^i_0)}{q_{t|0}(f^i|f^i_0)} p^{\theta}_{0|t}(f^i_0| \mathcal{G}_t)$, $\tilde{\mathbf{R}}^{(i,j)}_{\theta,t}(e^{(i,j)},\bar{e}^{(i,j)})=B_t^{(i,j)} \sum_{e_0^{(i,j)}}\frac{q_{t|0}(\bar{e}^{(i,j)}|e^{(i,j)}_0)}{q_{t|0}(e^{(i,j)}|e^{(i,j)}_0)} p^{\theta}_{0|t}(e^{(i,j)}_0| \mathcal{G}_t)$. If we present the permutation $\mathcal P$ on every node as a bijection $\sigma:\{1,\dots,n\}\mapsto \{1,\dots,n\}$, the term
    \begin{align}
        \tilde{\mathbf{R}}^i_{\theta,t}(f^i,\bar{f}^i) &= A_t^i \sum_{f_0^i}\frac{q_{t|0}(\bar{f}^i|f^i_0)}{q_{t|0}(f^i|f^i_0)} p^{\theta}_{0|t}(f^i_0| \mathcal{G}_t)\\
        &= \mathbf{R}^i_t(\bar{f}^i,f^i)\delta_{\bar{\mathcal{G}}\setminus \bar{f}^i, \mathcal{G}\setminus f^i}\sum_{f_0^i}\frac{q_{t|0}(\bar{f}^i|f^i_0)}{q_{t|0}(f^i|f^i_0)} p^{\theta}_{0|t}(f^i_0| \mathcal{G}_t)\\
        &\overset{(*)}{=} \mathbf{R}^{\sigma(i)}_t(\bar{f}^{\sigma(i)},f^{\sigma(i)})\delta_{\mathcal P(\bar{\mathcal{G}})\setminus \bar{f}^{\sigma(i)}, \mathcal P(\mathcal{G})\setminus f^{\sigma(i)}}\sum_{f_0^{\sigma(i)}}\frac{q_{t|0}(\bar{f}^{\sigma(i)}|f^{\sigma(i)}_0)}{q_{t|0}(f^{\sigma(i)}|f^{\sigma(i)}_0)} p^{\theta}_{0|t}(f^i_0| \mathcal{G}_t)\\
        &\overset{(**)}{=}\mathbf{R}^{\sigma(i)}_t(\bar{f}^{\sigma(i)},f^{\sigma(i)})\delta_{\mathcal P(\bar{\mathcal{G}})\setminus \bar{f}^{\sigma(i)}, \mathcal P(\mathcal{G})\setminus f^{\sigma(i)}}\sum_{f_0^{\sigma(i)}}\frac{q_{t|0}(\bar{f}^{\sigma(i)}|f^{\sigma(i)}_0)}{q_{t|0}(f^{\sigma(i)}|f^{\sigma(i)}_0)} p^{\theta}_{0|t}(f^{\sigma(i)}_0| \mathcal P(\mathcal{G}_t))\\
        &= \tilde{\mathbf{R}}^{\sigma(i)}_{\theta,t}(f^{\sigma(i)},\bar{f}^{\sigma(i)})
    \end{align}
    where (*) is based on the permutation invariant of the forward process and its rate matrix; (**) is based on the permutation equivariance of the graph-to-graph backbone $p^{\theta}_{0|t}$.
\end{proof}

\subsection{Proof of Lemma~\ref{lemma: permutation-invariant transition}}
\label{sec: proof of permutation-invariant transition}
Recall the Kolmogorov forward equation, for $s<t$,
\begin{align}
    \frac{d}{dt}q_{t|s}(\mathbf{x}_t|\mathbf{x}_s) = \sum_{\xi \in \mathcal{X}} q_{t|s}(\xi|\mathbf{x}_s)\mathbf{R}_t(\xi, \mathbf{x}_t).
    \label{eq: kolmogorov forward repeat}
\end{align}
\begin{proof}
    We aim to show that $q_{t|s}(\mathcal P(\mathbf{x}_t)|\mathcal P(\mathbf{x}_s))$ is a solution of Eq.~\eqref{eq: kolmogorov forward repeat}. Because the permutation $\mathcal P$ is a bijection, we have
    \begin{align}
       &\frac{d}{dt}q_{t|s}(\mathcal P(\mathbf{x}_t)|\mathcal P(\mathbf{x}_s))\\ =& \sum_{\xi \in \mathcal{X}} q_{t|s}(\mathcal P(\xi)|\mathcal P(\mathbf{x}_s))\mathbf{R}_t(\mathcal P(\xi), \mathcal P(\mathbf{x}_t))\\
       \overset{(*)}{=}& \sum_{\xi \in \mathcal{X}} q_{t|s}(\mathcal P(\xi)|\mathcal P(\mathbf{x}_s))\mathbf{R}_t(\xi, \mathbf{x}_t) \label{eq: kolmogorov forward permute}
    \end{align}
    where (*) is because $\mathbf R_t$ is permutation-invariant.
    As Eq.~\ref{eq: kolmogorov forward permute} and Eq.~\ref{eq: kolmogorov forward repeat} share the same rate matrix, and the rate matrix completely determines the CTMC (and its Kolmogorov forward equation)~\cite{resnick1992adventures}, thus, their solutions are the same: $q_{t|s}(\mathbf{x}_t|\mathbf{x}_s) = q_{t|s}(\mathcal P(\mathbf{x}_t)|\mathcal P(\mathbf{x}_s))$, i.e., the transition probability is permutation-invariant.
\end{proof}

\subsection{Proof of Theorem~\ref{thm: permutation-invariant density}}
\label{sec: proof of permutation density}
\begin{proof}
    We start from a simple case where the parametric rate matrix is fixed all the time,
   \begin{align}
    p_0^\theta(\mathcal G_0) = \sum_{\mathcal G_T}q^\theta_{0|T}(\mathcal G_0 | \mathcal G_T)\pi_{\texttt{ref}}(\mathcal G_T),
    \end{align}
    where the transition probability is by solving the Kolmogorov forward equation
    \begin{align}
    \frac{d}{dt}q^\theta_{t|s}(\mathcal{G}_t|\mathcal{G}_s) = \sum_{\xi} q^\theta_{t|s}(\xi|\mathcal{G}_s)\tilde{\mathbf{R}}_\theta(\xi, \mathcal G_t).
    \end{align}
    Thus, the sampling probability of permuted graph $\mathcal P(\mathcal G_0)$
    \begin{align}
        p_0^\theta(\mathcal P(\mathcal G_0)) &= \sum_{\mathcal G_T}q^\theta_{0|T}(\mathcal P(\mathcal G_0) | \mathcal P(\mathcal G_T))\pi_{\texttt{ref}}(\mathcal P(\mathcal G_T))\\
        &\overset{(*)}{=} \sum_{\mathcal G_T}q^\theta_{0|T}(\mathcal G_0 | \mathcal G_T)\pi_{\texttt{ref}}(\mathcal P(\mathcal G_T))\\
        &\overset{(**)}{=} \sum_{\mathcal G_T}q^\theta_{0|T}(\mathcal G_0 | \mathcal G_T)\pi_{\texttt{ref}}(\mathcal G_T)\\
        &=p_0^\theta(\mathcal G_0)
    \end{align}
    where (*) is based on Lemma~\ref{lemma: permutation-invariant rate} and Lemma~\ref{lemma: permutation-invariant transition}, the transition probability of \textsc{DisCo} is permutation-invariant and (**) is from the assumption that the reference distribution $\pi_{\texttt{ref}}(\mathcal G_T)$ is permutation-invariant. Thus, we proved that for the simple case, $\tilde{\mathbf{R}}_{\theta,t}$ fixed $\forall t$, the sampling probability is permutation-invariant.

    For the practical sampling, as we mentioned in Section~\ref{sec: sampling reverse process}, the $\tau$-leaping algorithm assumes that the time interval $[0,T]$ is divided into various length-$\tau$ intervals $[0,\tau), [\tau,2\tau),\dots,[T-\tau,T]$ (here both close sets or open sets work) and assume the reverse rate matrix is fixed as $\tilde{\mathbf{R}}_{\theta,t}$ within every length-$\tau$ interval, such as $(t-\tau, t]$. Thus, the sampling probability can be computed as
    \begin{align}
        p^{\theta}_0(\mathcal{G}_0) = \sum_{\mathcal{G}_{T}, \mathcal{G}_{T-\tau},\dots, \mathcal{G}_{\tau}}q_{0|\tau}(\mathcal{G}_0|\mathcal{G}_{\tau})\dots q_{T-\tau|T}(\mathcal{G}_{T-\tau}|\mathcal{G}_T)\pi_{\texttt{ref}}(\mathcal{G}_T).
    \end{align}
    The conclusion from the simple case can be generalized to this $\tau$-leaping-based case because all the transition probability $q_{t-\tau|t}(\mathcal{G}_{t-\tau}|\mathcal{G}_t)$ and the reference distribution are permutation-invariant.
\end{proof}
Note that Xu et al.~\cite{DBLP:conf/iclr/XuY0SE022} have a similar analysis in their Proposition 1 on a DDPM-based model.

\subsection{Proof of Theorem~\ref{thm: permutation-invariant training loss}}
\label{sec: proof of permutation loss}
Recall our training objective is
\begin{align}
    \min_\theta  T\mathbb{E}_{t\sim\mathcal{U}_{(0,T)}}\mathbb{E}_{\mathcal{G}_0}\mathbb{E}_{q_{t|0}(\mathcal{G}_t|\mathcal{G}_0)}\bigg[\sum_i \mathcal{L}_{\mathtt{CE}}(\mathtt{One\text-Hot}(f_0^i),\hat{f}_0^i) +  \sum_{i,j} \mathcal{L}_{\mathtt{CE}}(\mathtt{One\text-Hot}(e_0^{(i,j)}),\hat{e}_0^{(i,j)})\bigg]
\end{align}
where $\hat{f}^i_0 = [p^{\theta}_{0|t}(f^i=1 | \mathcal{G}_t),\dots, p^{\theta}_{0|t}(f^i=b | \mathcal{G}_t)]^\top\in[0,1]^b$ and $\hat{e}^{(i,j)}_0=[p^{\theta}_{0|t}(e^{(i,j)}=1| \mathcal{G}_t), \dots, p^{\theta}_{0|t}(e^{(i,j)}=a+1| \mathcal{G}_t)]^\top \in[0,1]^{a+1}$
\begin{proof}
    We follow the notation and present the permutation $\mathcal P$ on every node as a bijection $\sigma:\{1,\dots,n\}\mapsto \{1,\dots,n\}$. We first analyze the cross-entropy loss on the nodes for a single training graph $\mathcal G_0$ and taking expectation $\mathbb E_{\mathcal G_0}$ keeps the permutation invariance:
    \begin{align}
        \mathcal{L}_{\texttt{node}}(\mathcal G_0) &= T\mathbb{E}_{t\sim\mathcal{U}_{(0,T)}}\mathbb{E}_{q_{t|0}(\mathcal{G}_t|\mathcal{G}_0)}\sum_i \mathcal{L}_{\mathtt{CE}}(\mathtt{One\text-Hot}(f_0^i),\hat{f}_0^i)\\
        &= T\mathbb{E}_{t\sim\mathcal{U}_{(0,T)}}\sum_{\mathcal G_t}q_{t|0}(\mathcal{G}_t|\mathcal{G}_0)\sum_i \mathcal{L}_{\mathtt{CE}}(\mathtt{One\text-Hot}(f_0^i),\hat{f}_0^i)\\
        &\overset{(*)}{=} T\mathbb{E}_{t\sim\mathcal{U}_{(0,T)}}\sum_{\mathcal G_t}q_{t|0}(\mathcal P (\mathcal{G}_t)|\mathcal P (\mathcal{G}_0))\sum_i \mathcal{L}_{\mathtt{CE}}(\mathtt{One\text-Hot}(f_0^i),\hat{f}_0^i)\\
        &\overset{(**)}{=} T\mathbb{E}_{t\sim\mathcal{U}_{(0,T)}}\sum_{\mathcal G_t}q_{t|0}(\mathcal P (\mathcal{G}_t)|\mathcal P (\mathcal{G}_0))\sum_i \mathcal{L}_{\mathtt{CE}}(\mathtt{One\text-Hot}(f_0^{\sigma(i)}),\hat{f}_0^{\sigma(i)})\\
        &= \mathcal{L}_{\texttt{node}}(\mathcal P (\mathcal G_0))
    \end{align}
    where (*) is from the permutation invariance of the forward process and (**) is from the permutation equivariance of the graph-to-graph backbone and the permutation invariance of the cross-entropy loss. A similar result can be analyzed on the cross-entropy loss on the edges
    \begin{align}
        \mathcal{L}_{\texttt{edge}}(\mathcal G_0) &= T\mathbb{E}_{t\sim\mathcal{U}_{(0,T)}}\mathbb{E}_{q_{t|0}(\mathcal{G}_t|\mathcal{G}_0)}\sum_{i,j} \mathcal{L}_{\mathtt{CE}}(\mathtt{One\text-Hot}(e_0^{(i,j)}),\hat{e}_0^{(i,j)}) = \mathcal{L}_{\texttt{edge}}(\mathcal P(\mathcal G_0))
    \end{align}
    and we omit the proof here for brevity.
\end{proof}

\section{Sampling Algorithm}
\label{sec: algorithms appendix}

A Step-by-step procedure about the $\tau$-leaping graph generation is presented in Algorithm~\ref{alg: tau leaping}.

\begin{algorithm}
\caption{$\tau$-Leaping Graph Generation}
\label{alg: tau leaping}

\begin{algorithmic}[1]
\State $t\gets T$
\State $\mathcal{G}_t = (\{e^{(i,j)}\}_{i,j\in\mathbb{N}^{+}_{\leq n}},
\{f^i\}_{i\in\mathbb{N}^{+}_{\leq n}}) \gets \pi_{\texttt{ref}}(\mathcal{G})$
\While {$t>0$}

  \For{$i = 1, \dots, n$}
    \For{$s = 1, \dots, b$}
      \State $\tilde{\mathbf{R}}^i_{\theta,t}(f^i,s) = \mathbf{R}_t^i(s,f^i) \sum_{f_0^i}\frac{q_{t|0}(s|f^i_0)}{q_{t|0}(f^i|f^i_0)} p_{\theta}(f^i|\mathcal{G}_t,t)$
      \State $J_{f^i,s} \gets  \textrm{Poisson}(\tau\mathbf{R}_t^i(s,f^i))$ \Comment{\# of transition for every node}
    \EndFor
  \EndFor

  \For{$i,j = 1, \dots, n$}
      \For{$s = 1, \dots, a$}
        \State $\tilde{\mathbf{R}}^{(i,j)}_{\theta,t}(e^{(i,j)},s) = \mathbf{R}_t^{(i,j)}(s,e^{(i,j)}) \sum_{e_0^{(i,j)}}\frac{q_{t|0}(s|e^{(i,j)}_0)}{q_{t|0}(e^{(i,j)}|e^{(i,j)}_0)} p_{\theta}(e^{(i,j)}|\mathcal{G}_t,t)$
        \State $J_{e^{(i,j)},s} \gets  \textrm{Poisson}(\tau\mathbf{R}_t^{(i,j)}(s,e^{(i,j)}))$ \Comment{\# of transition for every edge}
      \EndFor
  \EndFor

  \For{$i = 1,\dots,n$}
    \If{$\sum_{s=1}^{b} J_{f^i,s} > 1$ or $\sum_{s=1}^{b} J_{f^i,s} = 0$}
    \State $f^i\gets f^i$ \Comment{stay the same}
    \Else
    \State $s^*=\arg\max_{s} \{J_{f^i,s}\}_{s=1}^b$
    \State $f^i\gets s^*$ \Comment{update node}
    \EndIf
  \EndFor

  \For{$i,j = 1,\dots,n$}
      \If{$\sum_{s=1}^{a} J_{e^{(i,j)},s} > 1$ or $\sum_{s=1}^{a} J_{e^{(i,j)},s} = 0$}
      \State $e^{(i,j)}\gets e^{(i,j)}$ \Comment{stay the same}
      \Else
      \State $s^*=\arg\max_{s} \{J_{e^{(i,j)},s}\}_{s=1}^b$
      \State $e^{(i,j)}\gets s^*$ \Comment{update edge}
      \EndIf
  \EndFor
  
  \State $t \gets t-\tau$
\EndWhile 
\end{algorithmic}
\end{algorithm}

\section{Auxiliary Features, PNA and FiLM Modules}
\label{sec: auxiliary features}
For learning a better graph-to-graph mapping $p_{0|t}^{\theta}(\mathcal{G}_{0}|\mathcal{G}_{t})$, artificially augmenting the node-level features and graph-level features is proved effective to enhance the expressiveness of graph learning models~\cite{DBLP:conf/aaai/YouGYL21,DBLP:conf/iclr/VignacKSWCF23}. For this setting, we keep consistent with the state-of-the-art model, DiGress~\cite{DBLP:conf/iclr/VignacKSWCF23}, and extract the following three sets of auxiliary features. Note that the following features are extracted on the noised graph $\mathcal{G}_{t}$.

We binarize the edge tensor $\mathbf{E}$ into an adjacency matrix $\mathbf{A}\in \{0,1\}^{n\times n}$ whose $1$ entries denote the corresponding node pair is connected by any type of edge.

\textbf{Motif features. }
The number of length-$3/4/5$ cycles every node is included in is counted as the topological node-level features; also, the total number of length-$3/4/5/6$ cycles is the topological graph-level features.

\textbf{Spectral features. }
The graph Laplacian is decomposed. The number of connected components and the first $5$ non-zero eigenvalues are selected as the spectral graph-level features. An estimated indicator of whether a node is included in the largest connected component and the first $2$ eigenvectors of the non-zero eigenvalues are selected as the spectral node-level features.

\textbf{Molecule features.} 
On molecule datasets, the valency of each atom is selected as the node-level feature, and the total weight of the whole molecule is selected as the graph-level feature.

The above node-level features and graph-level features are concatenated together as the auxiliary node-level features $\mathbf{F}_{\texttt{aux}}$ and graph-level features $\mathbf{y}$. An important property is that the above node-level features are permutation-equivariant and the above graph-level features are permutation-invariant, whose proof is straightforward so we omit it here.

Next, two important modules used in the MPNN backbone: PNA and FiLM are detailed.

\textbf{PNA module.} The PNA module~\cite{DBLP:conf/nips/CorsoCBLV20} is implemented as follows,
\begin{align}
    \mathtt{PNA}(\{\mathbf x_i\}_{i=1}^n) = \mathtt{MLP}(\mathtt{min}(\{\mathbf x_i\}_{i=1}^n)\oplus\mathtt{max}(\{\mathbf x_i\}_{i=1}^n)\oplus\mathtt{mean}(\{\mathbf x_i\}_{i=1}^n)\oplus\mathtt{std}(\{\mathbf x_i\}_{i=1}^n))
    \label{eq: PNA}
\end{align}
where $\oplus$ is the concatenation operator, $\mathbf x_i\in\mathbb R^d$; $\mathtt{min}$, $\mathtt{max}$, $\mathtt{mean}$, and $\mathtt{std}$ are coordinate-wise, e.g., $\mathtt{min}(\{\mathbf x_i\}_{i=1}^n)\in\mathbb R^d$.

\textbf{FiLM module.} FiLM~\cite{DBLP:conf/aaai/PerezSVDC18} is implemented as follows,
\begin{align}
    \mathtt{FiLM}(\mathbf x_i, \mathbf x_j) = \mathtt{Linear}(\mathbf x_i) + \mathtt{Linear}(\mathbf x_i)\odot \mathbf x_j + \mathbf x_j
    \label{eq: film}
\end{align}
where $\mathtt{Linear}$ is a single fully-connected layer without activation function and $\odot$ is the Hadamard product.
\section{Supplementary Details about Experiments}
\label{sec: experiments appendix}

\subsection{Hardware and Software}
\label{sec: hardware and implmentations}
We implement \textsc{DisCo} in PyTorch\footnote{\url{https://pytorch.org}} and PyTorch-geometric\footnote{\url{https://pytorch-geometric.readthedocs.io/en/latest}}. All the efficiency study results are from one NVIDIA Tesla V100 SXM2-32GB GPU on a server with 96 Intel(R) Xeon(R) Gold 6240R CPU @ 2.40GHz processors and 1.5T RAM. The training on QM9 and Community can be finished in 2 hours. For the training on SBM, Planar, it can be finished within 48 hours to get decent validity. The training on MOSES and GuacaMol can be finished within 96 hours.

\subsection{Dataset Setup}
\label{sec: dataset settings appendix}
\subsubsection{Dataset Statistics}
The statistics about all the datasets used in this paper are presented in Table~\ref{tab: dataset statistics}, where $a$ is the number of edge types, $b$ is the number of node types, $|\mathbf{E}|$ is the number of edges and $|\mathbf{F}|$ is the number of nodes.

\begin{table}[ht!]
\centering
\caption{Dataset statistics.}
\label{tab: dataset statistics}
\resizebox{\textwidth}{!}{%
\begin{tabular}{lllllllll}
\toprule
Name & \# Graphs & Split & $a$ & $b$ & Avg. $|\mathbf{E}|$ & Max $|\mathbf{E}|$ & Avg. $|\mathbf{F}|$ & Max $|\mathbf{F}|$ \\\midrule
SBM & 200 & 128/32/40 & 1 & 1 & 1000.8 & 2258 & 104.0 & 187 \\
Planar & 200 & 128/32/40 & 1 & 1 & 355.7 & 362 & 64.0 & 64 \\
Community & 100 & 64/16/20 & 1 & 1 & 74.0 & 122 & 15.7 & 20 \\
QM9 & 130831 & 97734/20042/13055 & 4 & 4 & 18.9 & 28 & 8.8 & 9 \\
MOSES & 1733214 & 1419512/156176/157526 & 4 & 8 & 46.3 & 62 & 21.6 & 27 \\
GuacaMol & 1398213 & 1118633/69926/209654 & 4 & 12 & 60.4 & 176 & 27.8 & 88 \\\bottomrule
\end{tabular}%
}
\end{table}

\subsubsection{Detailed Settings on Plain Graph Datasets}
\label{sec: detailed settings on plain data}
\paragraph{Dataset Split. }
We follow the settings of SPECTRE~\cite{DBLP:conf/icml/MartinkusLPW22} and DiGress~\cite{DBLP:conf/iclr/VignacKSWCF23} to split the SBM, Planar~\cite{DBLP:conf/icml/MartinkusLPW22}, and Community~\cite{DBLP:conf/icml/YouYRHL18} datasets into $64/16/20\%$ for training/validation/test set.

\paragraph{Metrics. }
The Maximum Mean Discrepancy (MMD)~\cite{DBLP:conf/icml/YouYRHL18} measures the discrepancy between two sets of distributions. The relative squared MMD~\cite{DBLP:conf/iclr/VignacKSWCF23}is defined as follows
\begin{align}
    score = \frac{\textrm{MMD}^2(\{\mathcal{G}\}_{\texttt{gen}}||\{\mathcal{G}\}_{\texttt{test}})}{\textrm{MMD}^2(\{\mathcal{G}\}_{\texttt{train}}||\{\mathcal{G}\}_{\texttt{test}})},
\end{align}
where $(\{\mathcal{G}\}_{\texttt{gen}}$, $(\{\mathcal{G}\}_{\texttt{train}}$, and $(\{\mathcal{G}\}_{\texttt{test}}$ are the sets of generated graphs, training graphs, and test graphs, respectively. We report the above relative squared MMD for degree distributions (Deg.), clustering coefficient distributions (Clus.), and average orbit counts (Orb.) statistics (the number of occurrences of all substructures with $4$ nodes). In addition, the Uniqueness, Novelty, and Validity are chosen. Uniqueness reports the fraction of the generated nonisomorphic graphs; Novelty reports the fraction of the generated graphs not isomorphic with any graph from the training set; Validity checks the fraction of the generated graphs following some specific rules. For the SBM dataset, we follow the validity check from~\cite{DBLP:conf/icml/MartinkusLPW22} whose core idea is to check whether real SBM graphs are statistically indistinguishable from the generated graphs; for the Planar dataset, we check whether the generated graphs are connected and are indeed planar graphs. Because the Community dataset does not have the Validity metric, we only report the Uniqueness, Novelty, and Validity results on the SBM and Planar datasets.

We report mean$\pm$std in $5$ runs.


\paragraph{Baseline methods.}
GraphRNN~\cite{DBLP:conf/icml/YouYRHL18}, GRAN~\cite{DBLP:conf/nips/LiaoLSWHDUZ19},
GG-GAN~\cite{krawczuk2020gg}, 
MolGAN~\cite{DBLP:journals/corr/abs-1805-11973}, SPECTRE~\cite{DBLP:conf/icml/MartinkusLPW22}, EDP-GNN~\cite{DBLP:conf/aistats/NiuSSZGE20},
GraphGDP~\cite{DBLP:conf/icdm/HuangS0FL22},
DiscDDPM~\cite{DBLP:journals/corr/abs-2210-01549},
EDGE~\cite{chen2023efficient}, ConGress~\cite{DBLP:conf/iclr/VignacKSWCF23},
DiGress~\cite{DBLP:conf/iclr/VignacKSWCF23} are chosen.

\subsubsection{Detailed Settings on Molecule Graph Datasets}
\label{sec: detailed settings on molecule data}
\paragraph{Dataset Split. }
We follow the split of QM9 from DiGress~\cite{DBLP:conf/iclr/VignacKSWCF23} and follow the split of MOSES~\cite{DBLP:journals/corr/abs-1811-12823} and GuacaMol~\cite{DBLP:journals/jcisd/BrownFSV19} according to their benchmark settings. Their statistics are presented in Table~\ref{tab: dataset statistics}.

\paragraph{Metrics. }
For QM9, Uniqueness, Novelty, and Validity are chosen as metrics. The first two are the same as introduced in Section~\ref{sec: detailed settings on plain data}. The Validity is computed by building a molecule with RdKit~\footnote{\url{https://www.rdkit.org/}} and checking if we can obtain a valid SMILES string from it.

For MOSES, the chosen metrics include Uniqueness, Novelty, Validity, Filters, Fréchet ChemNet Distance (FCD), Similarity to a nearest neighbor (SNN), and Scaffold similarity (Scaf), which is consistent with DiGress~\cite{DBLP:conf/iclr/VignacKSWCF23}. The official evaluation code~\footnote{\url{https://github.com/molecularsets/moses}} is used to report the performance.

For GuacaMol, the chosen metrics include Uniqueness, Novelty, Validity, KL Divergence, and Frećhet ChemNet Distance (FCD), which is consistent with DiGress~\cite{DBLP:conf/iclr/VignacKSWCF23}. The official evaluation code~\footnote{\url{https://github.com/BenevolentAI/guacamol}} is used to report the performance.

We report mean$\pm$std in $5$ runs except MOSES and GuacaMol, whose computations are too expensive to repeat multiple times.

\paragraph{Baseline methods.}
CharacterVAE~\cite{DBLP:journals/corr/Gomez-Bombarelli16}, GrammarVAE~\cite{DBLP:journals/corr/KusnerH16}, GraphVAE~\cite{DBLP:conf/icann/SimonovskyK18}, GT-VAE~\cite{DBLP:journals/corr/abs-2104-04345}, Set2GraphVAE~\cite{DBLP:conf/iclr/VignacF22}, GG-GAN~\cite{krawczuk2020gg},  MolGAN~\cite{DBLP:journals/corr/abs-1805-11973}, 
SPECTRE~\cite{DBLP:conf/icml/MartinkusLPW22}, GraphNVP~\cite{DBLP:journals/corr/abs-1905-11600}, GDSS~\cite{DBLP:conf/icml/JoLH22}, EDGE~\cite{chen2023efficient}, ConGress~\cite{DBLP:conf/iclr/VignacKSWCF23},
DiGress~\cite{DBLP:conf/iclr/VignacKSWCF23}, GRAPHARM~\cite{DBLP:conf/icml/KongCSZPZ23},VAE~\cite{gomez2018automatic}, JT-VAE~\cite{DBLP:conf/icml/JinBJ18}, GraphINVENT~\cite{mercado2021graph}, LSTM~\cite{segler2018generating}, NAGVAE~\cite{kwon2020compressed}, and MCTS~\cite{jensen2019graph} are chosen.

\subsection{Hyperparameter Settings}
\label{sec: hyperparameter settings}
\paragraph{Forward Diffusion Settings.}
As we introduced in Proposition~\ref{prop: rate matrices for uniform and marginal dist}, we tried two sets of rate matrices for the node and edge forward diffusion, so that the converged distribution is either uniform or marginal distribution. We found the marginal distribution leads to better results than the uniform distribution. Thus, the reference distribution is the marginal distribution for all the main results, except Tables~\ref{tab: ablation study GT} and~\ref{tab: ablation study MPNN}. The performance comparison between the marginal diffusion and uniform diffusion is presented in the ablation study in Sections~
\ref{sec: ablation study} and\ref{sec: ablation study appendix}. The $\beta(t)$ controls how fast the forward process converges to the reference distribution, which is set as $\beta(t)=\alpha\gamma^t log(\gamma)$, which is consistent with many existing works~\cite{DBLP:conf/nips/HoJA20, DBLP:conf/iclr/0011SKKEP21, DBLP:conf/nips/CampbellBBRDD22}. In our implementation, we assume the converged time $T=1$ and for the forward diffusion hyperparameters $(\alpha, \gamma)$ we tried two sets: $(1.0,5.0)$ and $(0.8,2.0)$ where the former one can ensure at $T=1$ the distribution is very close to the reference distribution, and the latter one does not fully corrupt the raw data distribution so the graph-to-graph model $p_{0|t}^{\theta}$ is easier to train.

\paragraph{Reverse Sampling Settings.}
The number of sampling steps is determined by $\tau$, which is $\texttt{round}(\frac{1}{\tau})$ if we set the converged time $T=1$. We select the number of sampling steps from $\{50, 100, 500\}$, which is much smaller the number of sampling steps of DiGress~\cite{DBLP:conf/iclr/VignacKSWCF23} from $\{500, 1000\}$. For the number of nodes $n$ in every generated graph, we compute a graph size distribution of the training set by counting the number of graphs for different sizes (and normalize the counting to sum it up to $1$). Then, we will sample the number of nodes from this graph size distribution for graph generation.

\paragraph{Neural Network Settings.}
For \textsc{DisCo}-GT, the parametric graph-to-graph model $p^{\theta}_{0|t}$ is graph transformer (GT). We use the exactly same GT architecture as DiGress~\cite{DBLP:conf/iclr/VignacKSWCF23} and adopt their recommended configurations~\footnote{\url{https://github.com/cvignac/DiGress/tree/main/configs/experiment}}. The reason is that this architecture is not our contribution, and setting the graph-to-graph model $p^{\theta}_{0|t}$ same can ensure a fair comparison between the discrete-time graph diffusion framework (from DiGress) and the continuous-time graph diffusion framework (from this work). For \textsc{DisCo}-MPNN, we search the number of MPNN layers from $\{3,5,8\}$, set all the hidden dimensions the same, and search it from $\{256, 512\}$. For both variants, the dropout is set as $0.1$, the learning rate is set as $2e^{-4}$, and the weight decay is set as $0$.

\begin{table}[t!]
\centering
\caption{Generation performance (mean$\pm$std) on the Community dataset.}
\label{tab: effectiveness comparison on community}
\begin{tabular}{llll}
\toprule
Model     & Deg.$\downarrow$ & Clus.$\downarrow$ & Orb.$\downarrow$ \\\midrule           
GraphRNN~\cite{DBLP:conf/icml/YouYRHL18}  & 4.0              & 1.7               & 4.0              \\
GRAN~\cite{DBLP:conf/nips/LiaoLSWHDUZ19}      & 3.0              & 1.6               & 1.0              \\
EDP-GNN~\cite{DBLP:conf/aistats/NiuSSZGE20} & 2.5 & 2.0 & 3.0 \\
GraphGDP~\cite{DBLP:conf/icdm/HuangS0FL22} & 2.0 & 1.1 & - \\
DiscDDPM~\cite{DBLP:journals/corr/abs-2210-01549} & 1.2 & \textbf{0.9} & 1.5\\
EDGE~\cite{chen2023efficient} & 1.0 & 1.0 & 2.0\\
GG-GAN~\cite{krawczuk2020gg}    & 4.0              & 3.1               & 8.0              \\
MolGAN~\cite{DBLP:journals/corr/abs-1805-11973}    & 3.0              & 1.9               & 1.0              \\
SPECTRE~\cite{DBLP:conf/icml/MartinkusLPW22}   & 0.5              & 2.7               & 2.0              \\
DiGress~\cite{DBLP:conf/iclr/VignacKSWCF23}   & 1.0              & \textbf{0.9}               & 1.0              \\
\rowcolor{Gray}\textsc{DisCo}-MPNN & 1.4\scriptsize$\pm$0.5              & \textbf{0.9\scriptsize$\pm$0.2}      & \textbf{0.9\scriptsize$\pm$0.3}     \\
\rowcolor{Gray}\textsc{DisCo}-GT   & \textbf{0.9\scriptsize$\pm$0.2}     & \textbf{0.9\scriptsize$\pm$0.3}               & 1.1\scriptsize$\pm$0.4             
\\\bottomrule
\end{tabular}
\end{table}

\subsection{Additional Results on Community}
\label{sec: plain graph results appendix}

Additional Community plain graph dataset results are in Table~\ref{tab: effectiveness comparison on community}. Our observation is consistent with the main content: both variants of \textsc{DisCo} are on par with, or even better than the SOTA general graph diffusion generative model, DiGress.

\subsection{Additional Ablation Study}
\label{sec: ablation study appendix}
Table~\ref{tab: ablation study MPNN} shows the ablation study of \textsc{DisCo}-MPNN on QM9 dataset. Our observations are consistent with the main content: (1) generally, the fewer sampling steps, the lower the generation quality but method's performance is robust in terms of the decreasing of sampling steps; (2) the marginal reference distribution is better than the uniform distribution, consistent with the observation from DiGress~\cite{DBLP:conf/iclr/VignacKSWCF23}.

\begin{table*}[t!]
\centering
\caption{Ablation study (mean$\pm$std\%) with MPNN backbone. V., U., and N. mean Valid, Unique, and Novel.}
\label{tab: ablation study MPNN}
\begin{tabular}{lllll}
\toprule
Ref. Dist. & Steps & Valid $\uparrow$ & V.U. $\uparrow$ & V.U.N. $\uparrow$ \\\midrule
\multirow{6}{*}{Marginal} & 500 & 98.9\scriptsize$\pm$0.7 & 98.7\scriptsize$\pm$0.5 & 68.7\scriptsize$\pm$0.2 \\
& 100 & 98.4\scriptsize$\pm$1.1 & 98.0\scriptsize$\pm$1.0 & 69.1\scriptsize$\pm$0.6 \\
& 30 & 97.7\scriptsize$\pm$1.2 & 97.5\scriptsize$\pm$0.8 & 70.4\scriptsize$\pm$1.1 \\
& 10 & 92.3\scriptsize$\pm$1.9 & 91.9\scriptsize$\pm$2.2 & 66.4\scriptsize$\pm$1.7 \\
& 5 & 88.8\scriptsize$\pm$3.3 & 87.1\scriptsize$\pm$2.8 & 67.3\scriptsize$\pm$2.9 \\
& 1 & 64.4\scriptsize$\pm$2.7 & 63.2\scriptsize$\pm$1.9 & 55.8\scriptsize$\pm$1.4 \\\midrule
\multirow{6}{*}{Uniform} & 500 & 93.5\scriptsize$\pm$1.7 & 93.2\scriptsize$\pm$1.1 & 64.9\scriptsize$\pm$1.0 \\
& 100 & 93.1\scriptsize$\pm$2.1 & 92.6\scriptsize$\pm$1.7 & 66.2\scriptsize$\pm$1.9 \\
& 30 & 87.1\scriptsize$\pm$1.8 & 86.8\scriptsize$\pm$1.1 & 64.0\scriptsize$\pm$1.0 \\
& 10 & 83.7\scriptsize$\pm$3.2 & 81.9\scriptsize$\pm$2.1 & 61.3\scriptsize$\pm$2.0 \\
& 5 & 81.5\scriptsize$\pm$2.9 & 75.4\scriptsize$\pm$3.4 & 64.6\scriptsize$\pm$2.3 \\
& 1 & 71.3\scriptsize$\pm$2.3 & 42.2\scriptsize$\pm$4.0 & 36.9\scriptsize$\pm$3.2 \\\bottomrule
\end{tabular}
\end{table*}

\subsection{Convergence Study}
\label{sec: convergence study}
Figure~\ref{fig: convergence study} shows the training loss of \textsc{DisCo}-GT and \textsc{DisCo}-MPNN on four datasets, whose X-axis is the number of iterations (i.e., the number of epochs $\times$ the number of training samples $/$ batch size). We found that overall the training losses converge smoothly on $4$ datasets.

\begin{figure*}[ht]
    \subfigure[SBM]{\includegraphics[width = 0.5\textwidth]{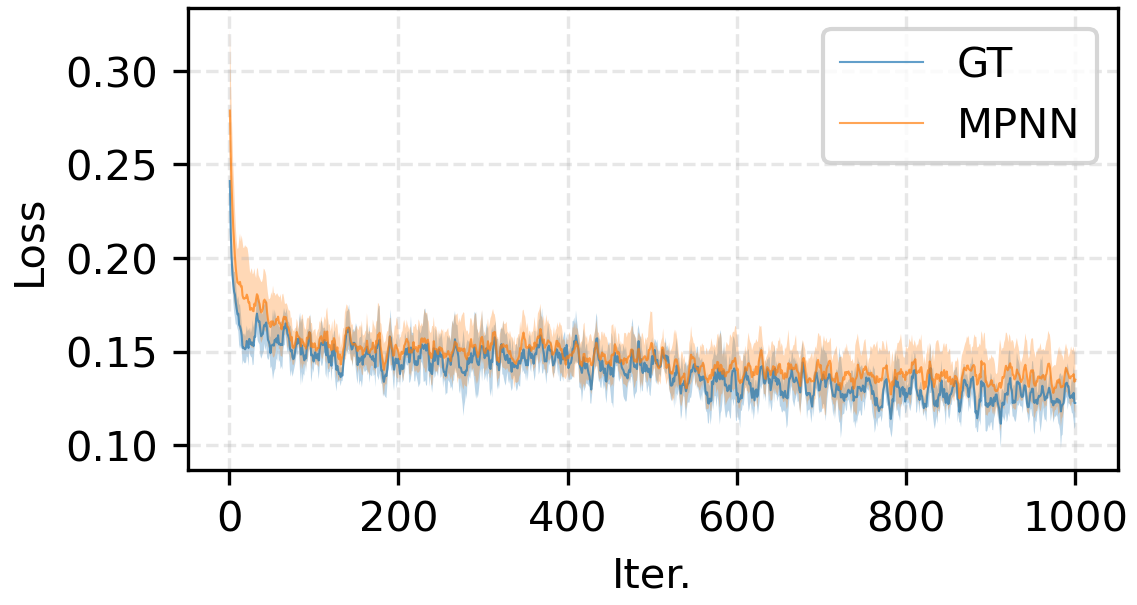}}
    \subfigure[Planar]{\includegraphics[width = 0.5\textwidth]{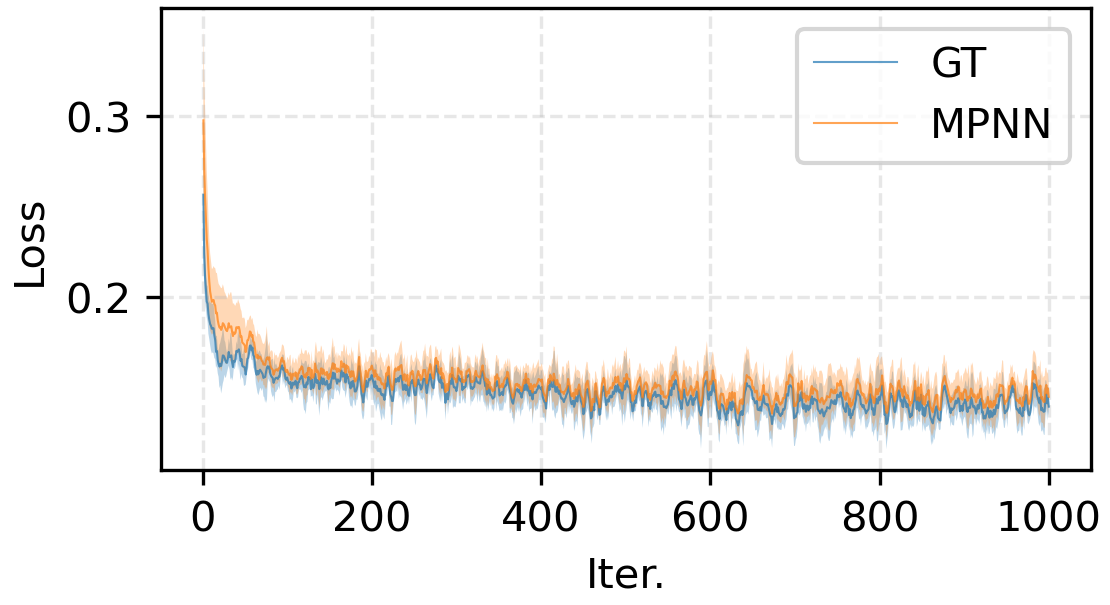}}
    \subfigure[Community]{\includegraphics[width = 0.5\textwidth]{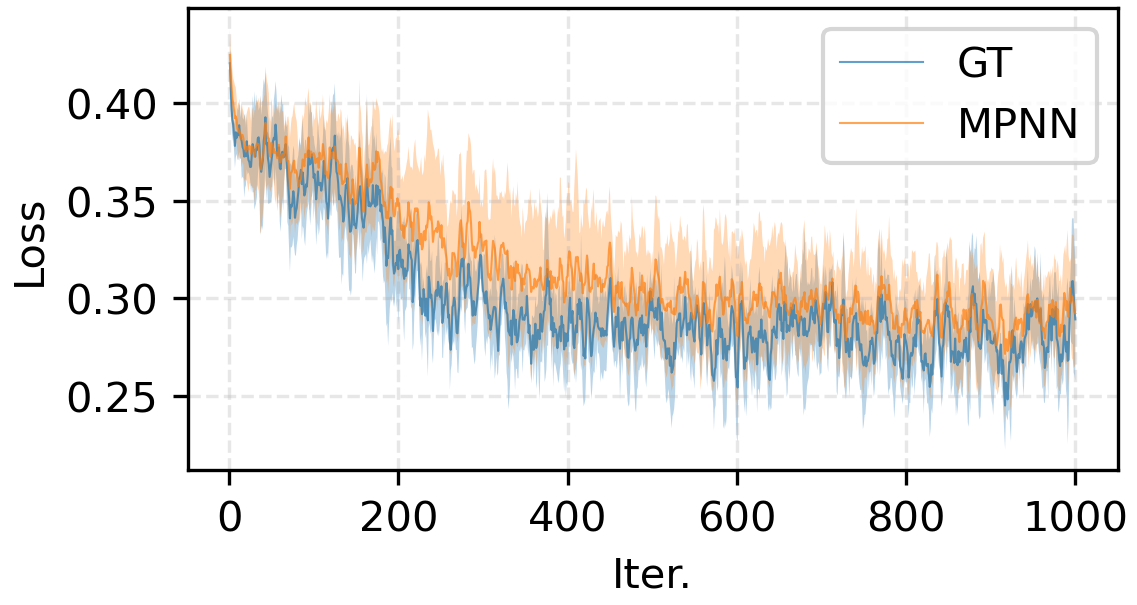}}
    \subfigure[QM9]{\includegraphics[width = 0.5\textwidth]{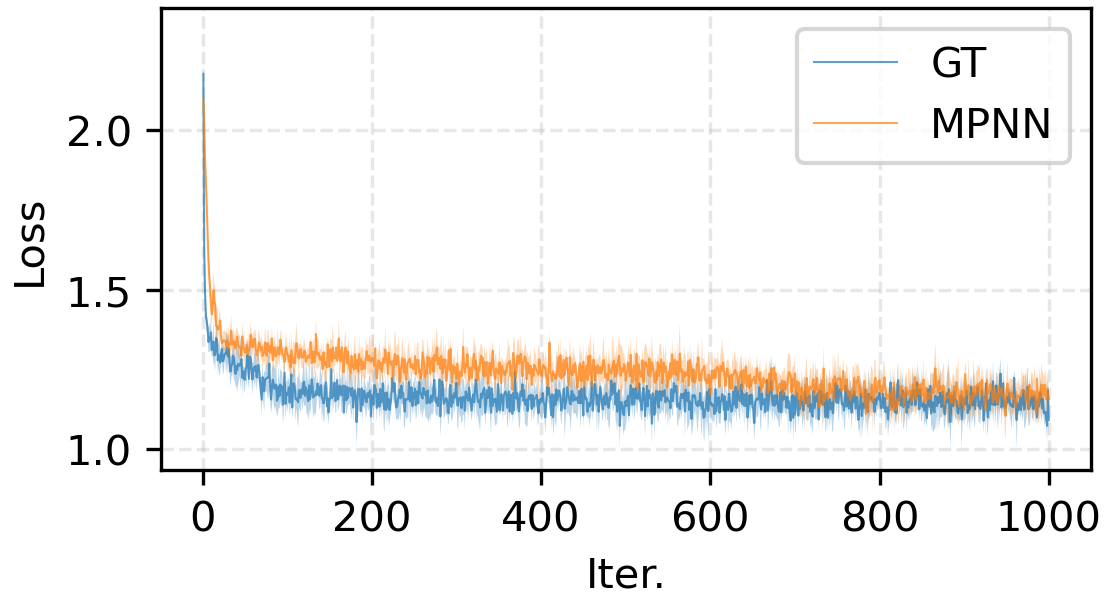}}
    \vspace{-5pt}
    \caption{Training loss of \textsc{DisCo} on different datasets and backbone models.}
    \label{fig: convergence study}
\vskip -0.2in
\end{figure*}

\subsection{Visualization}
\label{sec: visualization}
The generated graphs on the SBM and Planar datasets are presented in Figure~\ref{fig: generated graphs}. We clarify that the generated planar graphs are \emph{selected to be valid} because, as Table~\ref{tab: effectiveness comparison on plain graphs} shows, not all the generated graphs are valid planar graphs, but the planar layout can only visualize valid planar graphs in our setting~\footnote{\url{https://networkx.org/documentation/stable/reference/generated/networkx.drawing.layout.planar_layout.html}}. The generated SBM graphs are not selected; even if a part of them cannot pass the strict SBM statistic test (introduced in Section~\ref{sec: detailed settings on plain data} - Metrics), most, if not all, of them still form $2-5$ densely connected clusters.

The generation trajectory of SBM graphs is presented in Figure~\ref{fig: generation trajectory} which demonstrates the reverse denoising process visually.

\begin{figure*}[ht]
\vskip 0.2in
    \subfigure[SBM]{\includegraphics[width=\textwidth, trim=28mm 10mm 28mm 10mm, clip]{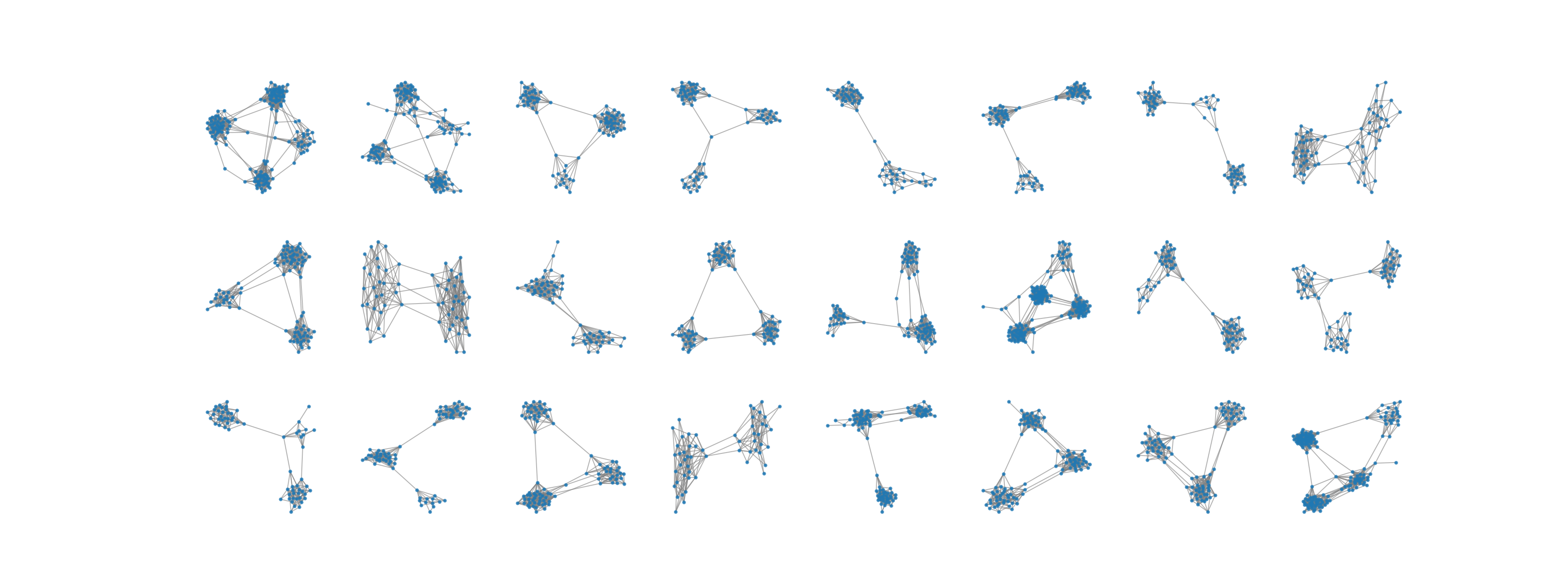}}
    \subfigure[Planar]{\includegraphics[width=\textwidth, trim=28mm 10mm 28mm 10mm, clip]{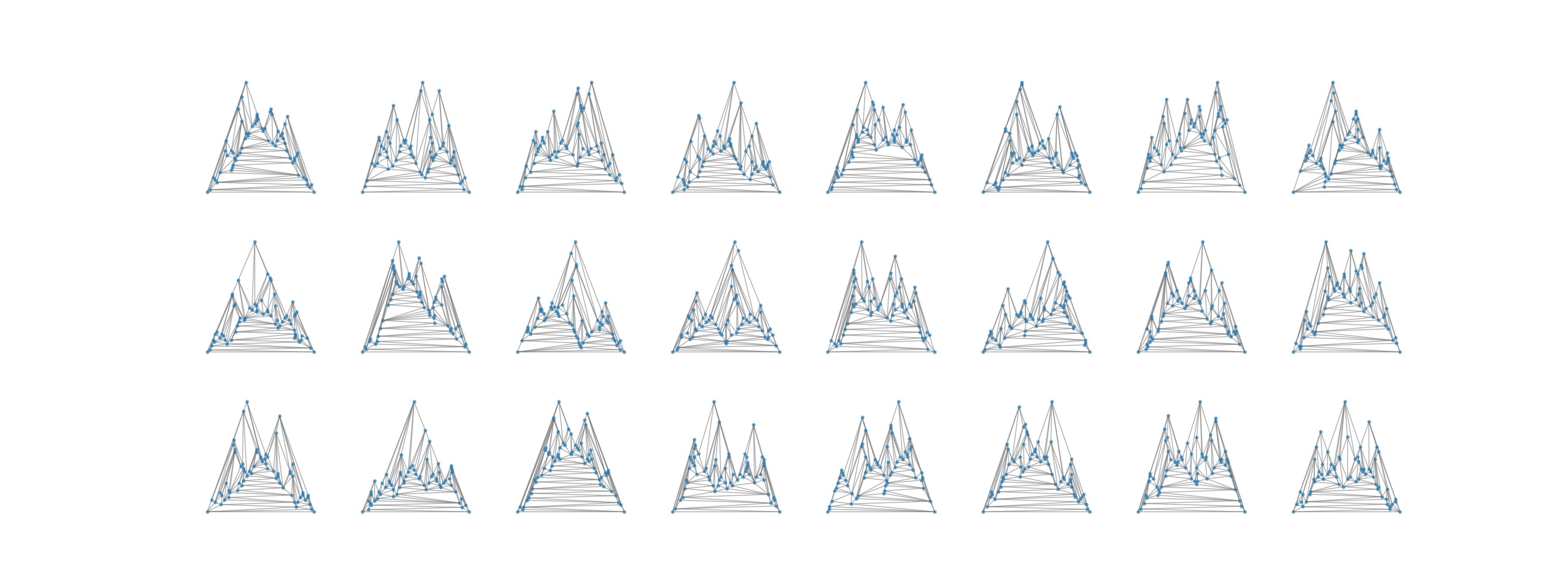}}
    \caption{Generated graphs.}
    \label{fig: generated graphs}
\vskip -0.2in
\end{figure*}

\begin{figure*}[t!]
\vskip 0.2in
\centering
\includegraphics[width=\textwidth, trim=55mm 40mm 55mm 40mm, clip]{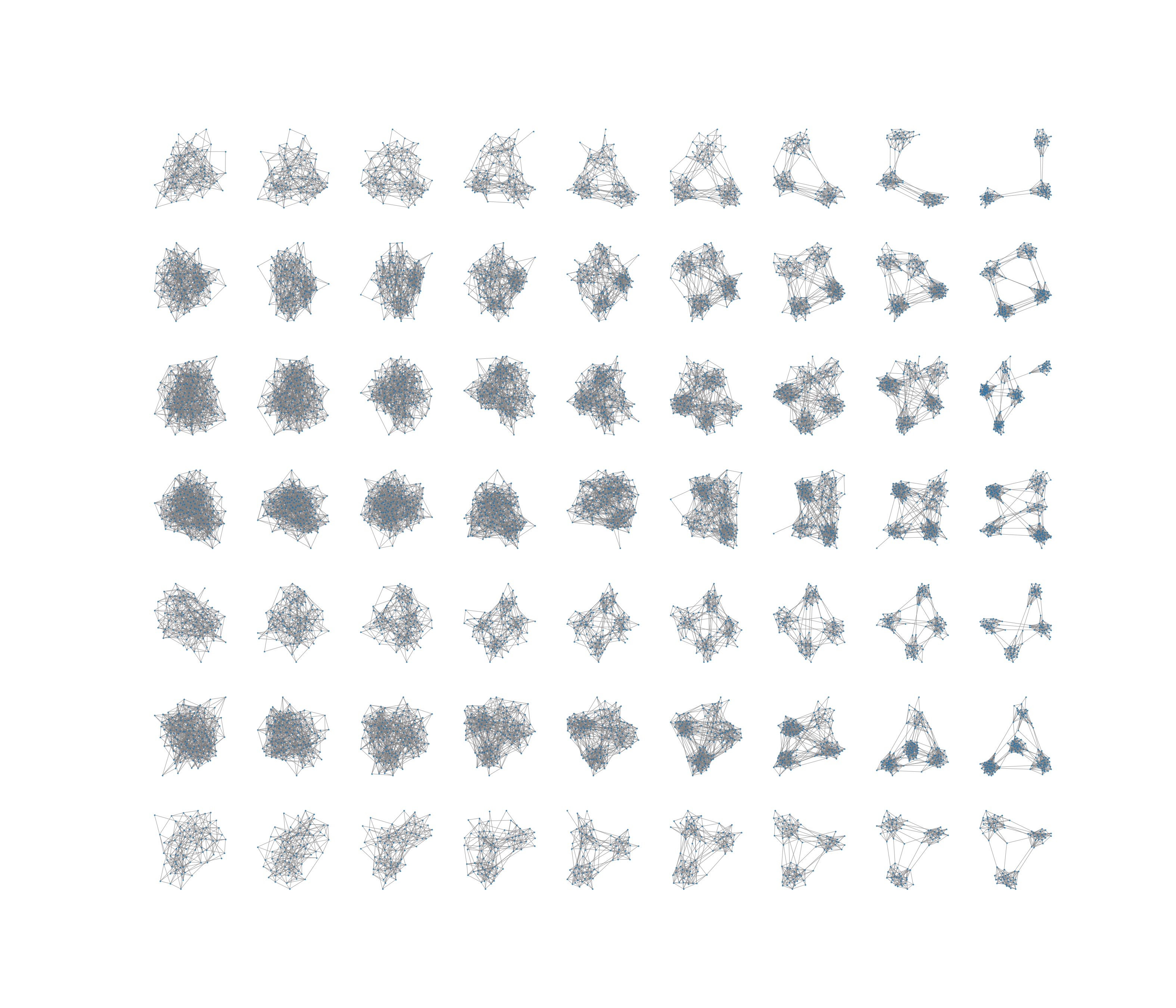}
\vspace{-4mm}
\caption{Generation trajectory of SBM graphs with different sizes. Every row is the generation trajectory of one graph from time $t=T$ (left) to $t=0$ (right) with equal time intervals.}
\label{fig: generation trajectory}
\end{figure*}

\section{Limitation and Future Work}
\label{sec: limitation}
In this paper, we study the generation of graphs with categorical node and edge types. The current model \textsc{DisCo} cannot be applied to generate graphs with multiple node/edge features (e.g., multiplex networks) and this is an important future work to study. Also, we view the absence of edge as a special type of edge, which forms a complete graph and promotes the expressiveness of our MPNN backbone model. However, it will lead to quadratic complexity concerning the number of nodes. For our current dataset (e.g. graphs with $<1000$ nodes) the complexity is still acceptable but for future studies on generating \emph{large} graphs, we aim to design more efficient diffusion generative models.

\section{Broader Impact}
\label{sec: broader impact}
This paper presents work whose goal is to advance the field of Machine Learning. There are many potential societal consequences of our work, none of which we feel must be specifically highlighted here.

\vfill\eject
\section*{NeurIPS Paper Checklist}

\begin{enumerate}

\item {\bf Claims}
    \item[] Question: Do the main claims made in the abstract and introduction accurately reflect the paper's contributions and scope?
    \item[] Answer: \answerYes{} 
    \item[] Justification: The abstract and introduction summarize all the theoretical and experimental contributions of this paper.

\item {\bf Limitations}
    \item[] Question: Does the paper discuss the limitations of the work performed by the authors?
    \item[] Answer: \answerYes{} 
    \item[] Justification: The limitation is mentioned in Section~\ref{sec: limitation}.

\item {\bf Theory Assumptions and Proofs}
    \item[] Question: For each theoretical result, does the paper provide the full set of assumptions and a complete (and correct) proof?
    \item[] Answer: \answerYes{} 
    \item[] Justification: Detailed assumptions and proofs are included in the Section~\ref{sec: proofs appendix}.

    \item {\bf Experimental Result Reproducibility}
    \item[] Question: Does the paper fully disclose all the information needed to reproduce the main experimental results of the paper to the extent that it affects the main claims and/or conclusions of the paper (regardless of whether the code and data are provided or not)?
    \item[] Answer: \answerYes{} 
    \item[] Justification: Detailed experimental settings are included in Section~\ref{sec: experiments appendix}. Also, the code of this paper is released in the supplementary materials.

\item {\bf Open access to data and code}
    \item[] Question: Does the paper provide open access to the data and code, with sufficient instructions to faithfully reproduce the main experimental results, as described in supplemental material?
    \item[] Answer: \answerYes{} 
    \item[] Justification: All the datasets are publicly available and their links are included in Section~\ref{sec: experiments appendix}. The code is released in the supplementary materials; we will formally release the code after acceptance.

\item {\bf Experimental Setting/Details}
    \item[] Question: Does the paper specify all the training and test details (e.g., data splits, hyperparameters, how they were chosen, type of optimizer, etc.) necessary to understand the results?
    \item[] Answer: \answerYes{} 
    \item[] Justification: All the training details are included in Section~\ref{sec: experiments appendix} and the released codes.

\item {\bf Experiment Statistical Significance}
    \item[] Question: Does the paper report error bars suitably and correctly defined or other appropriate information about the statistical significance of the experiments?
    \item[] Answer: \answerYes{} 
    \item[] Justification: We report average results with standard deviation on all the datasets, except MOSES and GuacaMol, whose computations are too expensive to repeat multiple times.

\item {\bf Experiments Compute Resources}
    \item[] Question: For each experiment, does the paper provide sufficient information on the computer resources (type of compute workers, memory, time of execution) needed to reproduce the experiments?
    \item[] Answer: \answerYes{} 
    \item[] Justification: The compute resources are detailed in Section~\ref{sec: hardware and implmentations}.
    
\item {\bf Code Of Ethics}
    \item[] Question: Does the research conducted in the paper conform, in every respect, with the NeurIPS Code of Ethics \url{https://neurips.cc/public/EthicsGuidelines}?
    \item[] Answer: \answerYes{} 
    \item[] Justification: We checked the code of Ethics and the paper conforms with the Code of Ethics. 

\item {\bf Broader Impacts}
    \item[] Question: Does the paper discuss both potential positive societal impacts and negative societal impacts of the work performed?
    \item[] Answer: \answerYes{} 
    \item[] Justification: It is mentioned in Section~\ref{sec: broader impact}. Our work is general graph generative modeling, which shares potential societal consequences with many established graph generative models.
    
\item {\bf Safeguards}
    \item[] Question: Does the paper describe safeguards that have been put in place for responsible release of data or models that have a high risk for misuse (e.g., pretrained language models, image generators, or scraped datasets)?
    \item[] Answer: \answerNA{} 
    \item[] Justification: We did not scrape any dataset from Internet and our released model is of low risk for misuse.

\item {\bf Licenses for existing assets}
    \item[] Question: Are the creators or original owners of assets (e.g., code, data, models), used in the paper, properly credited and are the license and terms of use explicitly mentioned and properly respected?
    \item[] Answer: \answerYes{} 
    \item[] Justification: All the datasets and codes of baseline methods are publicly available and for the academic purpose.

\item {\bf New Assets}
    \item[] Question: Are new assets introduced in the paper well documented and is the documentation provided alongside the assets?
    \item[] Answer: \answerNA{} 
    \item[] Justification: We did not release any new dataset and the code released will be well documented after the acceptance.

\item {\bf Crowdsourcing and Research with Human Subjects}
    \item[] Question: For crowdsourcing experiments and research with human subjects, does the paper include the full text of instructions given to participants and screenshots, if applicable, as well as details about compensation (if any)? 
    \item[] Answer: \answerNA{} 
    \item[] Justification: The paper does not involve crowdsourcing nor research with human subjects.

\item {\bf Institutional Review Board (IRB) Approvals or Equivalent for Research with Human Subjects}
    \item[] Question: Does the paper describe potential risks incurred by study participants, whether such risks were disclosed to the subjects, and whether Institutional Review Board (IRB) approvals (or an equivalent approval/review based on the requirements of your country or institution) were obtained?
    \item[] Answer: \answerNA{} 
    \item[] Justification: The paper does not involve crowdsourcing nor research with human subjects.

\end{enumerate}

\end{document}